\documentclass{article}

\usepackage[final]{neurips_2026}
\usepackage{hyperref}
\usepackage[utf8]{inputenc}
\usepackage[T1]{fontenc}
\usepackage{url}
\usepackage{microtype}
\usepackage{amsmath}
\usepackage{amsfonts}
\usepackage{amssymb}
\usepackage{amsthm}
\usepackage{algorithm}
\usepackage{algpseudocode}
\usepackage{booktabs}
\usepackage{float}
\usepackage{graphicx}
\usepackage{multirow}
\usepackage{pifont}
\usepackage{placeins}
\usepackage{tabularx}
\usepackage{xcolor}

\hypersetup{
  colorlinks=true,
  linkcolor=blue,
  urlcolor=blue,
  citecolor=green,
  pdfstartview=Fit,
  breaklinks=true
}

\newtheorem{proposition}{Proposition}

\title{STEPS: A Temporal Smooth Error Propagation Solver on the Manifolds for Test-Time Adaptation in Time Series Forecasting}
\author{
Jiaqi Liu$^{1,*}$ \quad
Yifan Ouyang$^{1,*}$ \quad
Zhifei Song$^{1}$ \quad
Sim Kuan Goh$^{1}$ \quad
Ashwaq Qasem$^{1,2,\dagger}$ \\
$^1$School of Artificial Intelligence and Robotics, Xiamen University Malaysia\\
$^2$S.M.A.R.T. NEXUS Centre of Excellence, Xiamen University Malaysia\\
Sepang, Selangor 43900, Malaysia\\
\texttt{ait2209080@xmu.edu.my} \quad
\texttt{ait2209084@xmu.edu.my} \quad
\texttt{ait2209089@xmu.edu.my}\\
\texttt{simkuan.goh@xmu.edu.my} \quad
\texttt{ashwaq.qasem@xmu.edu.my}\\
\small $^*$Equal contribution. \quad $^\dagger$Corresponding author: Ashwaq Qasem.
}

\begin{document}

\maketitle

\begin{abstract}
           Test-Time Adaptation (TTA) aims to improve time series forecasting under distribution shifts by using limited observations revealed during inference. However, forecasting TTA must operate in a source-free online setting, where the adaptation signal is short, temporally correlated, and potentially noisy. Existing methods can therefore suffer from weak identifiability, error accumulation, and unstable long-horizon corrections when the revealed prefix is sparse or contaminated. To address these issues, we propose STEPS, a Smooth Temporal Error Propagation Solver for TTA in time-series forecasting. STEPS reformulates forecasting TTA as a Dirichlet Boundary Value Problem on a temporal manifold, where the revealed prefix error serves as the boundary condition for the unknown future error field. Then, STEPS solves a smooth and bounded correction field in prediction space: a Local Solver propagates prefix errors under temporal smoothness, a Global Solver retrieves stable cross-window error memory and Spatiotemporal Manifold Fusion (SMF) integrates both solutions into the final correction. Across six standard benchmarks and four frozen backbones, STEPS achieves an average relative MSE reduction of 26.82\% over the zero-shot backbone, exceeding the strongest compared TTA baseline by 12.77\%. Additional sparse prefix and contamination tests confirm the robustness of STEPS under limited and noisy prefixes.
           
		\end{abstract}

\section*{Keywords}
Time Series Forecasting; Test-Time Adaptation; Distribution Shift; Forecast Error Correction; Robust Forecasting; Noisy Observations

	\section{Introduction}
	Time-series forecasting models are commonly deployed under temporal distribution shift. Changes in seasonality, operating regimes, sensor conditions, market dynamics, or weather patterns can alter the input-output relation after training\cite{kim2021reversible, liu2022nonstationary, dai2024ddn, ye2024frequency, qin2024evolving, zhao2024proactive}. As a result, even strong modern forecasting backbones\cite{kim2024self, woo2024unified} may exhibit systematic residual patterns during inference.
	
	Test-Time Adaptation (TTA) aims to exploit target-domain information observed at inference time without retraining the full source model\cite{zhang2024test, karmanov2024efficient, yuan2024tea, ma2024improved, shin2024tta, jia2024tinytta}. Forecasting TTA differs from vision-centric TTA in an important way: after a prediction is issued, ground truth values become observable with a short delay. This delayed supervision makes the residual between the zero-shot forecast and the revealed observations a direct test-time signal. Recent forecasting-oriented methods commonly operationalize this prefix as a small test-time supervision set: it can fit calibration rules, drive lightweight adaptation modules, or estimate an output-space correction for the current forecast~\cite{tafas2025, petsa2025, cosa2026}. Related online and normalization-based approaches improve robustness under non-stationary streams through model-side normalization or continual updates\cite{qin2024evolving, zhao2024proactive, wu2026test}.
	
	This training-oriented route makes delayed observations practically useful, but it also makes adaptation depend on a very small, temporally correlated, and occasionally corrupted prefix. When this prefix is treated as ordinary fitting data, sparse evidence or outliers can dominate the update, amplify short-term noise, and provide too little support for learning the best long-horizon correction pattern. The core difficulty is therefore not whether the prefix residual is informative, but how it should constrain the unobserved part of the future error trajectory.
	
	We formulate this problem from a boundary-value perspective. Rather than using the revealed prefix as a small dataset for online fitting, we view it as boundary evidence for an unknown future residual field. The prediction horizon is represented as a one-dimensional temporal manifold; observed prefix errors specify boundary values, while unobserved steps are inferred through a smooth and bounded extension in prediction space. This recasts forecasting TTA as a Dirichlet Boundary Value Problem over the output trajectory, where sparse or unreliable prefix evidence constrains the solution without being allowed to freely determine the whole horizon.
	
	Based on this formulation, we propose \textbf{STEPS}, a Smooth Temporal Error Propagation Solver for forecasting TTA. STEPS freezes the forecasting backbone and estimates only an output-space correction. Its local solver propagates prefix residuals through a low-Dirichlet-energy temporal extension, its global memory decoder retrieves recurring long-range residual patterns, and Spatiotemporal Manifold Fusion (SMF) combines both responses under an explicit correction bound. The resulting correction field uses the available prefix signal while controlling extrapolation beyond sparse or unreliable evidence.
	
	Our main contributions are:
	\begin{enumerate}
		\item \textbf{A Dirichlet Boundary Value Problem formulation.} We reformulate forecasting TTA as solving an unknown future error field on the prediction-horizon manifold, where the revealed prefix error provides Dirichlet boundary values rather than labels for online backbone fitting.
		\item \textbf{A Smooth Temporal Error Propagation Solver.} STEPS constructs an output-space correction through Prefix Boundary Selection, Local Solver: Smooth Propagation, Global Solver: Memory Retrieval, and Spatiotemporal Manifold Fusion (SMF), while keeping the forecasting backbone frozen.
		\item \textbf{Robust adaptation from limited prefix evidence.} Across six benchmarks and four frozen backbones, STEPS reduces MSE by 26.82\% over the zero-shot backbone on average and exceeds the strongest compared TTA baseline by 12.77 percentage points; sparse-prefix and prefix-contamination studies further show that STEPS can use limited boundary evidence without amplifying unreliable prefix errors.
	\end{enumerate}

	\section{Related Work}

\paragraph{Test-time adaptation.}
Test-time adaptation (TTA) studies how a deployed model can adjust to target-domain data at inference time, usually without access to the source training set. Early test-time training used auxiliary self-supervised objectives on test samples\cite{sun2020test}. Tent later showed that the fully test-time setting can be handled by entropy minimization over normalization parameters\cite{wang2021tent}, while source-free adaptation transferred a pretrained source hypothesis without revisiting source data\cite{liang2020shot}. As deployment scenarios moved from fixed target batches to online streams, repeated updates introduced pseudo-label uncertainty, model drift, and stability concerns\cite{wang2022cotta, wang2022note, niu2022eata}. Recent work addresses these issues by improving adaptation objectives, reducing the number of updated parameters, or adding memory for dynamic streams\cite{zhang2024test, yuan2024tea, ma2024improved, shin2024tta, jia2024tinytta, yuan2023rotta, gong2023sotta}. This line of work treats test-time observations as signals for updating or regularizing a deployed model; forecasting TTA inherits this idea, but its delayed labels require a different view of what the test-time signal should supervise.

\paragraph{TTA in non-stationary forecasting.}
Time-series forecasting is a natural but distinct TTA setting because target-domain evidence arrives in temporal order and the prediction horizon is itself structured. Before explicit forecasting TTA, non-stationarity was mainly addressed by making forecasters less sensitive to temporal shifts, for example through reversible normalization, stationarization, invariant temporal representations, dynamic normalization, frequency-aware normalization, or online non-stationary learning\cite{kim2021reversible, liu2022nonstationary, du2021adarnn, liu2024time, dai2024ddn, ye2024frequency, qian2024efficient}. These methods improve robustness to changing trend, scale, and seasonality, but they do not fully exploit the fact that forecasting targets are eventually revealed after prediction. Forecasting-specific TTA uses this delayed supervision as a small test-time fitting set: TAFAS calibrates input and output distributions with gated adaptation\cite{tafas2025}; PETSA makes such adaptation parameter-efficient\cite{petsa2025}; COSA uses recent residual context for direct output-space correction\cite{cosa2026}. These methods show that the prefix can guide adaptation, but their supervision view leaves the prefix vulnerable to sparsity and outliers and does not explicitly define how its evidence should extend across unobserved horizon steps. STEPS instead treats the prefix as boundary evidence for an error field, so the revealed residual constrains a smooth and bounded correction rather than serving as ordinary fitting data.

\paragraph{Smoothness and boundary-value propagation.}
A separate line of work explains why limited observations can be propagated through geometric structure. Manifold learning methods use neighborhood graphs to recover low-dimensional organization in high-dimensional data\cite{tenenbaum2000global, belkin2003laplacian, coifman2006diffusion}, and manifold regularization turns this idea into a smoothness preference for learned functions\cite{belkin2006manifold}. Graph-based semi-supervised learning gives a closely related operational view: observed labels or boundary values can be extended to unobserved nodes by minimizing graph-Laplacian variation, as in harmonic extension, consistency-based propagation, and Poisson learning\cite{zhu2003semi, zhou2003learning, calder2020poisson, holtz2024continuous}. Time-series graph models usually apply graphs to dependencies among variables or sensors\cite{wu2020connecting, marisca2024graph, cini2023graph}. Our STEPS draws on this principle at the level of the forecasting residual field.

Figure~\ref{fig:overall_framework} gives the overall solver pipeline used in the following methodology section.

\begin{figure*}[t]
	\centering
	\includegraphics[width=\textwidth]{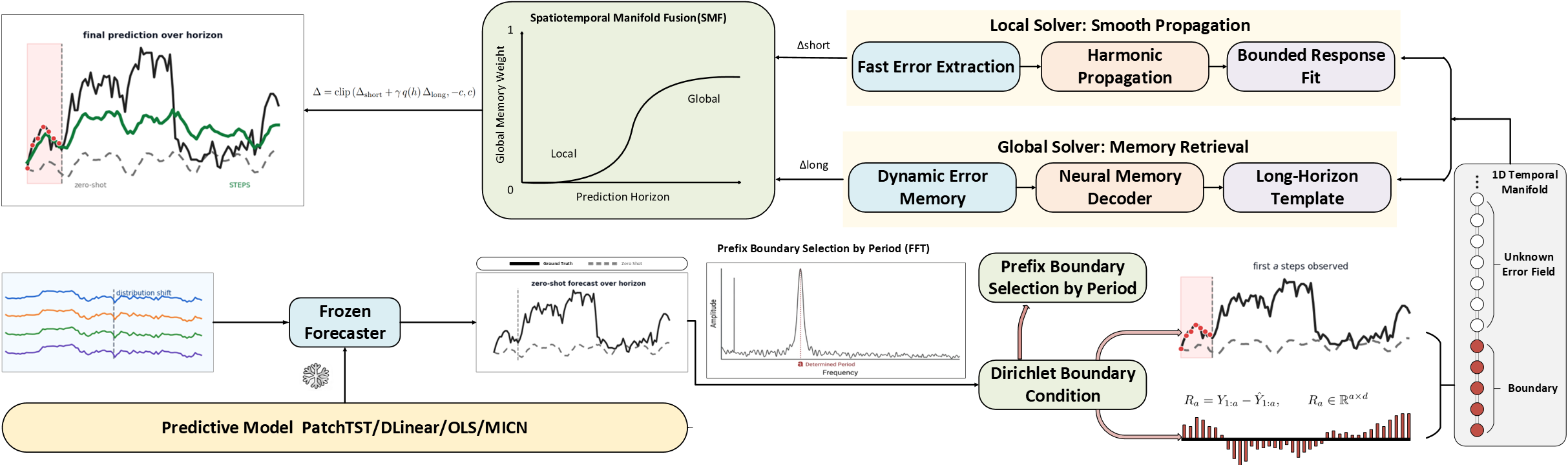}
	\caption{Overall framework of STEPS. A frozen forecaster produces the zero-shot trajectory, Prefix Boundary Selection by Period (FFT) converts the revealed prefix error into a Dirichlet Boundary Condition on the 1D temporal manifold, Local Solver: Smooth Propagation and Global Solver: Memory Retrieval estimate short- and long-horizon correction fields, and Spatiotemporal Manifold Fusion (SMF) forms the bounded final prediction.}
	\label{fig:overall_framework}
\end{figure*}

\section{Methodology}

\subsection{From TTA to a Dirichlet Boundary Value Problem}

A Dirichlet boundary value problem recovers an unknown field from observed boundary values by selecting a smooth extension. STEPS uses this view to estimate the future residual field while keeping the forecasting backbone frozen. Given $X\in\mathbb{R}^{L\times d}$, the backbone produces
\begin{equation}
	\hat{Y}=f_\theta(X), \qquad \theta \ \text{frozen}, \qquad \hat{Y}\in\mathbb{R}^{H\times d}.
\end{equation}
When the first $a$ future steps become visible during rollout, their discrepancy from the zero-shot forecast gives the prefix error
\begin{equation}
	R_a = Y_{1:a}-\hat{Y}_{1:a}, \qquad R_a\in\mathbb{R}^{a\times d}.
	\label{eq:prefix_error}
\end{equation}
The future horizon is a one-dimensional temporal manifold whose nodes are horizon steps; predictions, targets, and errors are fields on it. The prefix $R_a$ fixes the boundary values, and STEPS estimates the remaining error field under the \textbf{Smooth Error Manifold Assumption}: after the frozen backbone captures the dominant temporal structure, the residual component supported by the prefix should admit a low-variation extension unless the prefix suggests otherwise.

STEPS, the \textbf{Smooth Temporal Error Propagation Solver}, constructs a bounded output-space correction $\Delta\in\mathbb{R}^{H\times d}$ and applies it directly:
\begin{equation}
	\hat{Y}_{\mathrm{STEPS}}=\hat{Y}+\Delta.
\end{equation}
Its pipeline has four named steps: Prefix Boundary Selection, Local Solver: Smooth Propagation, Global Solver: Memory Retrieval, and Spatiotemporal Manifold Fusion (SMF).

\subsection{Prefix Boundary Selection}

The Prefix Boundary Selection by Period block turns delayed observations into current-window boundary evidence. The prefix error observes the current backbone deviation and serves as the Dirichlet boundary signal, rather than a label set for online parameter fitting. STEPS selects $a$ by estimating the dominant input period with FFT and exposing the corresponding prefix, tying the boundary to a meaningful temporal scale.

For the global decoder, the sparse prefix error is represented as a masked full-horizon tensor:
\begin{equation}
	\tilde{R}_{h} =
	\begin{cases}
		R_h, & h\le a,\\
		0, & h>a,
	\end{cases}
	\qquad
	m_h=\mathbb{I}[h\le a].
\end{equation}
The pair $(\tilde{R},m)$ marks observed and unobserved error positions for the global decoder.

\subsection{Local Solver: Smooth Propagation}

The local branch solves the prefix-identifiable part of the Dirichlet view. It propagates the selected boundary through fast error extraction, harmonic propagation, and bounded bias-aware response fitting.

\paragraph{Fast error extraction.}
STEPS removes low-frequency drift before propagating local deviations:
\begin{equation}
	R_{\mathrm{fast}}
	=
	R_a-\Pi_{\mathrm{low}}(R_a),
\end{equation}
where $\Pi_{\mathrm{low}}$ is a least-squares projection onto constant and linear temporal components.

The fast error is extended by harmonic propagation:
\begin{equation}
	\Delta_{\mathrm{harm}}
	=
	P_{:,1:a}R_{\mathrm{fast}},
	\qquad
	P=(D^\top D+\alpha I)^{-1}.
	\label{eq:local_harmonic_solver}
\end{equation}
where $D$ is the first-order temporal difference matrix and $D^\top D$ is the graph Laplacian of the one-dimensional temporal chain. This selects a low-variation correction field with discrete Dirichlet energy
\begin{equation}
	\sum_{h=1}^{H-1}\|\Delta_{h+1}-\Delta_h\|_2^2,
\end{equation}

\paragraph{Local bias field.}
To capture systematic over- or under-prediction, STEPS constructs a local bias field:
\begin{equation}
	\Delta_{\mathrm{bias}}
	=
	\mathbf{1}_{H}\,\mu_a^\top,
	\qquad
	\mu_a=\frac{1}{a}\sum_{h=1}^{a} R_h.
\end{equation}

\paragraph{Bounded response fitting.}
STEPS combines harmonic and bias fields by solving a small ridge problem on the observed prefix:
\begin{equation}
	\beta^\star
	=
	\arg\min_{\beta}
	\left\|
	B_{1:a}\beta - R_a
	\right\|_2^2
	+
	\lambda_{\mathrm{ridge}}\|\beta\|_2^2,
	\qquad
	B=[\Delta_{\mathrm{harm}},\Delta_{\mathrm{bias}}].
\end{equation}
The coefficients are clipped to avoid over-amplifying sparse evidence:
\begin{equation}
	\Delta_{\mathrm{short}}
	=
	\lambda_{\mathrm{short}} B\,
	\mathrm{clip}(\beta^\star,-b,b).
\end{equation}
This closed-form branch has no trainable neural weights and uses only the prefix error.

\subsection{Global Solver: Memory Retrieval}

The global branch supplies information not identifiable from the current boundary alone, such as persistent bias, delayed phase error, and long-range drift. It maintains Global Error Memory and decodes a full-horizon response conditioned on the current forecast and prefix evidence.

\paragraph{Dynamic error memory.}
STEPS maintains Global Error Memory only over completed rollout windows whose full targets have already been observed:
\begin{equation}
	M_t
	=
	\rho M_{t-1}
	+
	(1-\rho)\bar{R}_t,
	\qquad
	\bar{R}_t=\frac{1}{|\mathcal{B}_t|}\sum_{i\in\mathcal{B}_t}\left(Y_i-\hat{Y}_i\right).
	\label{eq:dynamic_memory}
\end{equation}
Here $M_t\in\mathbb{R}^{H\times d}$ is the full-horizon error template memory, $\rho$ is the forgetting factor, and $\mathcal{B}_t$ is the current evaluated batch. To avoid target leakage, the memory update is performed after the corresponding forecast window has been fully evaluated; the updated memory is used only for subsequent windows and never for correcting the unobserved future steps of the same window.

\paragraph{Neural error memory decoder.}
Dynamic memory stores a full-horizon template whose response is conditioned by the current forecast, prefix boundary, and local correction. STEPS uses a neural error-memory decoder:
\begin{equation}
	\Delta_{\mathrm{long}}
	=
	g_\phi
	\left(
	\hat{Y},
	\Delta_{\mathrm{short}},
	\tilde{R},
	m,
	M_t,
	z_t
	\right),
	\label{eq:neural_memory_decoder}
\end{equation}
where $z_t$ is a compact context vector of recent error statistics and $\phi$ denotes decoder parameters. The decoder is trained offline, frozen during test-time adaptation, and predicts only error tendencies rather than $Y$ itself.

\subsection{Spatiotemporal Manifold Fusion (SMF)}

SMF is the final output-space solver. It keeps corrections local near the observed boundary, increases memory influence toward distant horizons, and applies an explicit magnitude bound. The horizon-aware ramp is
\begin{equation}
	q(h)=\sigma\left(\kappa\left(\frac{h}{H}-\tau\right)\right),
\end{equation}
where $\kappa$ controls sharpness and $\tau$ controls the transition point. The final correction is
\begin{equation}
	\Delta
	=
	\mathrm{clip}
	\left(
	\Delta_{\mathrm{short}}
	+
	\gamma q(h)\Delta_{\mathrm{long}},
	-c,c
	\right),
	\label{eq:steps_response}
\end{equation}
where $\Delta_{\mathrm{short}}$ is the local response, $\Delta_{\mathrm{long}}$ is the global memory response, $\gamma$ controls the global contribution, and $c$ is the safety bound. The ramp emphasizes local propagation near the prefix and Global Error Memory at distant horizons, while clipping keeps the correction bounded under sparse or contaminated boundary evidence.

\section{Experiments}

We evaluate STEPS from three perspectives: average forecasting accuracy, stability under corrupted prefix evidence, and correction quality when only sparse test-time observations are available. The first setting follows the standard long-horizon forecasting protocol, while the latter two stress the reliability of the boundary-value formulation.

\subsection{Experimental Setup}

We evaluate on six standard multivariate forecasting benchmarks used by recent forecasting TTA work: ETTh1, ETTh2, ETTm1, ETTm2, Exchange, and Weather. STEPS is applied to frozen DLinear~\cite{zeng2023dlinear}, PatchTST~\cite{nie2023patchtst}, OLS~\cite{toner2024ols}, and MICN~\cite{wang2023micn} forecasters with look-back length $L=96$ and horizons $H\in\{96,192,336,720\}$. During test-time rollout, only the revealed prefix $Y_{1:a}$ is used to form $R_a=Y_{1:a}-\hat{Y}_{1:a}$; the remaining future labels are never used for adaptation. We compare with the zero-shot backbone and forecasting TTA baselines TAFAS~\cite{tafas2025}, PETSA~\cite{petsa2025}, COSA-F, and COSA-P~\cite{cosa2026}. For all TTA methods, backbone parameters remain frozen.

We report mean squared error (MSE) as the primary metric and mean absolute error (MAE) in diagnostic tables; lower is better. Relative improvement is computed as $(\mathrm{MSE}_{\mathrm{base}}-\mathrm{MSE}_{\mathrm{method}})/\mathrm{MSE}_{\mathrm{base}}$ against the corresponding zero-shot backbone. STEPS trains only the global error-memory decoder offline; test-time adaptation performs a forward pass plus closed-form local propagation and SMF clipping, without online backpropagation or backbone weight updates. Unless otherwise stated, experiments are run with PyTorch 2.5.1 and CUDA 12.1 on an NVIDIA GeForce RTX 4090 GPU.

\subsection{Main Forecasting Accuracy Across Horizons}

The main forecasting comparison is reported in Table~\ref{tab:main_results_4_backbones_steps}.

Table~\ref{tab:main_results_4_backbones_steps} reports the main forecasting comparison across datasets, horizons, and frozen backbones. STEPS achieves the strongest overall average gain without modifying the backbone, and the improvement is consistent across both simple linear models and stronger neural forecasters. This supports the central claim that smooth output-space error propagation can provide reliable test-time correction across diverse forecasting settings.

\subsection{Robustness and Component Analysis}

\paragraph{Clean component ablation.}
We isolate the main terms in Eq.~\ref{eq:steps_response} on DLinear across six datasets and all four prediction horizons. For each dataset-horizon pair, we train the full STEPS solver once and evaluate all ablated variants by disabling components only at test time. This shared-checkpoint protocol attributes the performance change to the removed solver term rather than to separate re-training.
The aggregated ablation is shown in Table~\ref{tab:component_ablation_rfmd}.

Table~\ref{tab:component_ablation_rfmd} separates the roles of the main solver terms. Local-only correction remains useful near the prefix but weakens as the horizon grows, where the boundary alone becomes less informative. Global Error Memory supplies cross-window error templates for those distant nodes, and removing memory consistently hurts long-horizon correction. The unbounded variant is slightly stronger under clean prefixes because it follows the observed boundary more aggressively; this is useful for interpreting the next robustness test, where the same aggressiveness becomes a liability.

\paragraph{Contaminated prefix robustness.}

We then corrupt the boundary evidence itself. During inference, a fraction of the visible prefix observations is replaced by outliers with magnitude $6\sigma$, while MSE is computed against the clean future target. All compared variants use the same revealed-prefix interface, so differences reflect how strongly each method trusts noisy boundary evidence.
The contamination results are summarized in Table~\ref{tab:prefix_outlier_robustness}.

Table~\ref{tab:prefix_outlier_robustness} shows the trade-off clearly. The unbounded STEPS variant is competitive when the boundary is clean, but degrades faster once revealed errors are contaminated. Full STEPS has the smallest degradation among adaptive methods, indicating that SMF uses the prefix without allowing a few corrupted boundary values to dominate the correction field. More aggressive prefix-fitting baselines react more strongly to the same anomalies and suffer larger error growth.

\subsection{Generalization Under Sparse Test-Time Observations}

Finally, we restrict the revealed boundary to only a few observations. This setting tests whether STEPS can still extract useful correction from sparse evidence without turning a small prefix into an unstable future trajectory.

\paragraph{Three-point sparse boundary.}

Table~\ref{tab:real_small_support_k3_only} reports the three-point sparse-boundary test.

Table~\ref{tab:real_small_support_k3_only} uses only $k=3$ revealed future points. We report near-field response and far-field preservation separately because a sparse boundary can improve nearby steps while still causing distant drift. STEPS is conservative near the boundary and avoids severe far-horizon degradation on most datasets, consistent with treating sparse observations as boundary evidence for a smooth field rather than as enough supervision to freely fit a future trajectory.

\paragraph{Sparse anchors in a predicted support window.}
We also evaluate sparse anchors inside a predicted support window. The first 36 forecasted steps form the support region, but only a small percentage of positions are replaced by true observations; all remaining positions keep the zero-shot predictions and therefore provide zero residual.
Table~\ref{tab:sparse_anchor_limited} reports this sparse-anchor protocol.

Table~\ref{tab:sparse_anchor_limited} shows the same behavior in a more practical sparse-anchor protocol. With only two to seven true anchors inside the support window, STEPS obtains the best average MSE across all tested anchor ratios. Together with the contamination results, this indicates that the solver can use limited boundary evidence while remaining stable when that evidence is incomplete or noisy.

\FloatBarrier

\section{Conclusion}

We presented \textbf{STEPS}, a Smooth Temporal Error Propagation Solver for test-time adaptation in time-series forecasting. STEPS reformulates forecasting TTA as a Dirichlet Boundary Value Problem on a temporal manifold, where revealed prefix errors act as boundary evidence for a smooth and bounded output-space correction. By combining Local Smooth Propagation, Global Error Memory Retrieval, and Spatiotemporal Manifold Fusion, STEPS improves frozen-backbone forecasts without online parameter updates and remains robust under sparse or contaminated prefix evidence.

\paragraph{Limitations and future work.}
The current formulation still has limitations. Its correction quality depends on the reliability of the revealed prefix and the relevance of accumulated error memory. When the target stream undergoes an extremely abrupt regime change, the visible prefix may no longer provide a stable boundary for the unobserved future; similarly, when historical rollout windows contain weak recurrence, the global memory branch may contribute less reliable long-range templates. STEPS also uses a fixed temporal geometry and a fixed horizon-aware fusion schedule, which may be suboptimal for streams whose characteristic periods or uncertainty levels change rapidly. Future work will explore adaptive temporal geometry, uncertainty-aware boundary weighting, and more flexible long-horizon fusion for broader streaming deployment.

\begingroup
\setlength{\textfloatsep}{4pt plus 1pt minus 1pt}
\setlength{\floatsep}{4pt plus 1pt minus 1pt}
\setlength{\intextsep}{4pt plus 1pt minus 1pt}

\begin{table*}[t]
\centering
\caption{Prediction accuracy comparison on standard forecasting benchmarks. The 96/192/336/720 rows report MSE, where lower is better. The Imp. row reports average relative MSE improvement over Baseline.}
\label{tab:main_results_4_backbones_steps}
\tiny
\setlength{\tabcolsep}{1.2pt}
\renewcommand{\arraystretch}{0.82}
\resizebox{\textwidth}{!}{
\begin{tabular}{cc*{12}{c}}
\toprule
& & \multicolumn{6}{c}{\textbf{DLinear}} & \multicolumn{6}{c}{\textbf{PatchTST}} \\
\cmidrule(lr){3-8}\cmidrule(lr){9-14}
\textbf{Data} & \textbf{H} & Base & TAFAS & PETSA & COSA-F & COSA-P & STEPS & Base & TAFAS & PETSA & COSA-F & COSA-P & STEPS \\
\midrule
\multirow{5}{*}{ETTh1} & 96 & .4695 & .4618 & .4594 & .4574 & \textbf{.4482} & \underline{\textbf{.3754}} & .4312 & .4262 & .4269 & .4242 & .4238 & \underline{\textbf{.3533}} \\
 & 192 & .5213 & .5117 & .5118 & .5066 & .5050 & \underline{\textbf{.4366}} & .4955 & .4865 & .4854 & .4830 & .4805 & \underline{\textbf{.4085}} \\
 & 336 & .5659 & .5604 & .5617 & .5528 & .5456 & \underline{\textbf{.4781}} & .5559 & .5478 & .5475 & .5438 & .5320 & \underline{\textbf{.4408}} \\
 & 720 & .7117 & .6820 & .6743 & .6107 & .5896 & \underline{\textbf{.5797}} & .7117 & .6860 & .6822 & .6113 & .5822 & \underline{\textbf{.5459}} \\
 & Imp. & - & +2.16\% & +2.49\% & +5.48\% & +7.10\% & \underline{\textbf{+17.59\%}} & - & +2.01\% & +2.17\% & +5.11\% & +6.81\% & \underline{\textbf{+19.91\%}} \\
\midrule
\multirow{5}{*}{ETTh2} & 96 & .2323 & .2303 & .2306 & .2300 & .2281 & \underline{\textbf{.1572}} & .2362 & .2351 & .2362 & .2349 & .2343 & \underline{\textbf{.1625}} \\
 & 192 & .2862 & .2842 & .2876 & .2827 & .2819 & \underline{\textbf{.2040}} & .2826 & .2758 & .2773 & .2665 & .2608 & \underline{\textbf{.1996}} \\
 & 336 & .3252 & .3185 & .3184 & .3050 & .3083 & \underline{\textbf{.2283}} & .3199 & .3125 & .3132 & .2971 & .2978 & \underline{\textbf{.2164}} \\
 & 720 & .4087 & .3873 & .3853 & .3062 & .3477 & \underline{\textbf{.2867}} & .4264 & .4005 & .4012 & .3233 & .3428 & \underline{\textbf{.3119}} \\
 & Imp. & - & +2.21\% & +2.01\% & +8.38\% & +5.86\% & \underline{\textbf{+30.17\%}} & - & +2.81\% & +2.47\% & +9.39\% & +8.76\% & \underline{\textbf{+29.94\%}} \\
\midrule
\multirow{5}{*}{ETTm1} & 96 & .3715 & .3497 & .3524 & .3456 & .3475 & \underline{\textbf{.2556}} & .4024 & .3894 & .3937 & .3625 & .3626 & \underline{\textbf{.2624}} \\
 & 192 & .4438 & .4166 & .4178 & .4113 & .4122 & \underline{\textbf{.3557}} & .4512 & .4372 & .4413 & .4250 & .4258 & \underline{\textbf{.3789}} \\
 & 336 & .5183 & .4799 & .4803 & .4753 & .4858 & \underline{\textbf{.4127}} & .5081 & .4905 & .4946 & \textbf{.4568} & .4697 & \underline{\textbf{.4142}} \\
 & 720 & .5929 & .5488 & .5532 & \textbf{.4774} & .4991 & .4885 & .5629 & .5427 & .5462 & \textbf{.4681} & .4882 & .4696 \\
 & Imp. & - & +6.71\% & +6.26\% & \textbf{+10.52\%} & +8.92\% & \underline{\textbf{+22.26\%}} & - & +3.35\% & +2.49\% & +10.67\% & +9.09\% & \underline{\textbf{+21.47\%}} \\
\midrule
\multirow{5}{*}{ETTm2} & 96 & .1598 & .1584 & .1584 & .1583 & .1586 & \underline{\textbf{.1211}} & .1584 & .1581 & .1583 & .1558 & .1562 & \underline{\textbf{.1136}} \\
 & 192 & .1930 & .1913 & .1913 & .1904 & .1905 & \underline{\textbf{.1731}} & .2059 & .2036 & .2037 & .2007 & .2022 & \underline{\textbf{.1708}} \\
 & 336 & .2324 & .2289 & .2292 & .2083 & .2242 & \underline{\textbf{.1832}} & .2458 & .2451 & .2452 & \textbf{.2258} & .2352 & \underline{\textbf{.1926}} \\
 & 720 & .3062 & .2968 & .2963 & \textbf{.2215} & .2316 & .2507 & .3268 & .3268 & .3256 & \textbf{.2446} & .2645 & .2803 \\
 & Imp. & - & +1.58\% & +1.59\% & +10.08\% & +7.48\% & \underline{\textbf{+18.46\%}} & - & +0.40\% & +0.44\% & +9.36\% & +6.64\% & \underline{\textbf{+20.31\%}} \\
\midrule
\multirow{5}{*}{Exchange} & 96 & .0913 & .0885 & .0878 & .0812 & .0834 & \underline{\textbf{.0636}} & .0867 & .0843 & .0837 & .0765 & .0788 & \underline{\textbf{.0566}} \\
 & 192 & .1827 & .1760 & .1730 & .1459 & .1519 & \underline{\textbf{.1237}} & .1877 & .1805 & .1832 & .1464 & .1570 & \underline{\textbf{.1173}} \\
 & 336 & .3277 & .2941 & .2920 & .2039 & .2480 & \underline{\textbf{.1488}} & .3389 & .3275 & .3300 & \textbf{.1983} & .2445 & \underline{\textbf{.1709}} \\
 & 720 & .8873 & .8762 & .8781 & \textbf{.3494} & .4481 & .4215 & .8648 & .8659 & .8643 & \textbf{.3543} & .4662 & .5756 \\
 & Imp. & - & +4.56\% & +5.27\% & \textbf{+32.40\%} & +24.83\% & \underline{\textbf{+42.41\%}} & - & +2.46\% & +2.14\% & +33.57\% & +24.85\% & \underline{\textbf{+38.80\%}} \\
\midrule
\multirow{5}{*}{Weather} & 96 & .1954 & .1796 & .1823 & .1773 & .1793 & \underline{\textbf{.1192}} & .1742 & .1724 & .1743 & .1624 & .1634 & \underline{\textbf{.1275}} \\
 & 192 & .2403 & .2244 & .2254 & .2216 & .2217 & \underline{\textbf{.1621}} & .2195 & .2147 & .2167 & .2006 & .2108 & \underline{\textbf{.1663}} \\
 & 336 & .2918 & .2709 & .2740 & .2567 & .2626 & \underline{\textbf{.1912}} & .2766 & .2666 & .2701 & .2451 & .2488 & \underline{\textbf{.2250}} \\
 & 720 & .3643 & .3500 & .3497 & .2581 & .2708 & \underline{\textbf{.2354}} & .3544 & .3383 & .3442 & \textbf{.2590} & .2713 & .2755 \\
 & Imp. & - & +6.45\% & +5.75\% & +14.56\% & +12.91\% & \underline{\textbf{+35.36\%}} & - & +2.84\% & +1.61\% & +13.42\% & +10.92\% & \underline{\textbf{+22.99\%}} \\
\midrule
\midrule
& & \multicolumn{6}{c}{\textbf{OLS}} & \multicolumn{6}{c}{\textbf{MICN}} \\
\cmidrule(lr){3-8}\cmidrule(lr){9-14}
\textbf{Data} & \textbf{H} & Base & TAFAS & PETSA & COSA-F & COSA-P & STEPS & Base & TAFAS & PETSA & COSA-F & COSA-P & STEPS \\
\midrule
\multirow{5}{*}{ETTh1} & 96 & .4511 & .4409 & .4391 & .4390 & .4372 & \underline{\textbf{.3556}} & .5103 & .4901 & .4898 & .4693 & .4684 & \underline{\textbf{.4164}} \\
 & 192 & .5046 & .4934 & .4937 & .4915 & .4906 & \underline{\textbf{.4176}} & .5954 & .5617 & .5620 & .5372 & \textbf{.5328} & \underline{\textbf{.4949}} \\
 & 336 & .5510 & .5440 & .5465 & .5385 & .5320 & \underline{\textbf{.4545}} & .6615 & .6387 & .6420 & .5950 & .5878 & \underline{\textbf{.5671}} \\
 & 720 & .6997 & .6630 & .6431 & .5969 & .5733 & \underline{\textbf{.5597}} & .9233 & .8142 & .8375 & .7001 & .6504 & \underline{\textbf{.5487}} \\
 & Imp. & - & +2.75\% & +3.43\% & +5.56\% & +6.84\% & \underline{\textbf{+18.98\%}} & - & +6.22\% & +5.47\% & +13.01\% & +14.86\% & \underline{\textbf{+22.53\%}} \\
\midrule
\multirow{5}{*}{ETTh2} & 96 & .2306 & .2285 & .2288 & .2232 & .2265 & \underline{\textbf{.1547}} & .2582 & .2551 & .2552 & .2492 & \textbf{.2485} & \underline{\textbf{.1874}} \\
 & 192 & .2839 & .2824 & .2848 & .2796 & .2791 & \underline{\textbf{.1996}} & .3282 & .3179 & .3258 & .3049 & \textbf{.3017} & \underline{\textbf{.2369}} \\
 & 336 & .3258 & .3182 & .3189 & .3003 & .3043 & \underline{\textbf{.2253}} & .3732 & .3482 & .3497 & \textbf{.3241} & .3310 & \underline{\textbf{.2939}} \\
 & 720 & .4162 & .3908 & .3884 & .3177 & .3453 & \underline{\textbf{.2873}} & .4617 & .4474 & .4473 & \textbf{.3650} & .3885 & .3867 \\
 & Imp. & - & +2.47\% & +2.32\% & +9.05\% & +6.78\% & \underline{\textbf{+31.11\%}} & - & +3.53\% & +2.83\% & \textbf{+11.17\%} & +9.75\% & \underline{\textbf{+23.19\%}} \\
\midrule
\multirow{5}{*}{ETTm1} & 96 & .3710 & .3506 & .3536 & .3454 & .3475 & \underline{\textbf{.2553}} & .4354 & .3951 & .3951 & .3837 & \textbf{.3831} & \underline{\textbf{.3202}} \\
 & 192 & .4439 & .4160 & .4184 & .4115 & .4119 & \underline{\textbf{.3544}} & .4855 & .4566 & .4574 & \textbf{.4476} & .4514 & \underline{\textbf{.4177}} \\
 & 336 & .5182 & .4787 & .4792 & .4748 & .4749 & \underline{\textbf{.4120}} & .5556 & .5108 & .5082 & \textbf{.4832} & .5054 & .5111 \\
 & 720 & .5922 & .5478 & .5522 & \textbf{.4763} & .5007 & .4904 & .6212 & .5756 & .5778 & \textbf{.5029} & .5225 & .5499 \\
 & Imp. & - & +6.73\% & +6.18\% & +10.54\% & +9.34\% & \underline{\textbf{+22.25\%}} & - & +7.65\% & +7.64\% & \textbf{+12.94\%} & +10.99\% & \underline{\textbf{+14.98\%}} \\
\midrule
\multirow{5}{*}{ETTm2} & 96 & .1602 & .1590 & .1589 & .1582 & .1586 & \underline{\textbf{.1206}} & .1710 & .1711 & .1730 & \textbf{.1702} & .1704 & \underline{\textbf{.1106}} \\
 & 192 & .1936 & .1921 & .1919 & .1906 & .1907 & \underline{\textbf{.1758}} & .2121 & \textbf{.2102} & .2126 & \textbf{.2102} & .2120 & \underline{\textbf{.1641}} \\
 & 336 & .2331 & .2299 & .2302 & .2131 & .2226 & \underline{\textbf{.1816}} & .2530 & .2501 & .2520 & \textbf{.2337} & .2351 & \underline{\textbf{.1900}} \\
 & 720 & .3066 & .2986 & .2971 & \textbf{.2171} & .2349 & .2548 & .3327 & .3220 & .3131 & \textbf{.2477} & .2643 & .2600 \\
 & Imp. & - & +1.38\% & +1.51\% & +10.14\% & +7.60\% & \underline{\textbf{+18.22\%}} & - & +1.30\% & +1.22\% & \textbf{+8.64\%} & +7.01\% & \underline{\textbf{+26.18\%}} \\
\midrule
\multirow{5}{*}{Exchange} & 96 & .0814 & .0792 & .0798 & .0756 & .0773 & \underline{\textbf{.0625}} & .1151 & .1087 & .1146 & \textbf{.0955} & .1008 & \underline{\textbf{.0780}} \\
 & 192 & .1727 & .1658 & .1653 & .1393 & .1457 & \underline{\textbf{.1202}} & .2150 & .2198 & .1999 & \textbf{.1663} & .1722 & \underline{\textbf{.1450}} \\
 & 336 & .3226 & .2877 & .2898 & .2020 & .2323 & \underline{\textbf{.1549}} & .3950 & .3047 & .3100 & \textbf{.2119} & .2660 & \underline{\textbf{.1900}} \\
 & 720 & .8366 & .8138 & .8149 & \textbf{.3444} & .4541 & .4276 & 1.0259 & .7191 & .7805 & \textbf{.3871} & .4815 & .4850 \\
 & Imp. & - & +5.06\% & +4.75\% & +30.67\% & +23.60\% & \underline{\textbf{+38.62\%}} & - & +14.02\% & +13.22\% & \textbf{+37.08\%} & +29.51\% & \underline{\textbf{+42.35\%}} \\
\midrule
\multirow{5}{*}{Weather} & 96 & .1957 & .1807 & .1795 & .1772 & .1803 & \underline{\textbf{.1192}} & .1757 & .1853 & .1970 & \textbf{.1636} & .1651 & \underline{\textbf{.1200}} \\
 & 192 & .2404 & .2244 & .2274 & .2223 & .2237 & \underline{\textbf{.1623}} & .2237 & .2161 & .2265 & \textbf{.2082} & .2120 & \underline{\textbf{.1640}} \\
 & 336 & .2921 & .2714 & .2748 & .2551 & .2642 & \underline{\textbf{.1915}} & .2812 & .2746 & .2788 & \textbf{.2729} & .2737 & \underline{\textbf{.1930}} \\
 & 720 & .3644 & .3466 & .3493 & .2579 & .2708 & \underline{\textbf{.2373}} & .3508 & .3573 & .3681 & \textbf{.2582} & .2855 & \underline{\textbf{.2380}} \\
 & Imp. & - & +6.57\% & +5.94\% & +14.72\% & +12.51\% & \underline{\textbf{+35.22\%}} & - & -0.39\% & -4.36\% & \textbf{+10.79\%} & +8.14\% & \underline{\textbf{+30.48\%}} \\
\bottomrule
\end{tabular}}
\end{table*}

\begin{table*}[!tbp]
\centering
\caption{Shared-checkpoint component ablation averaged over six datasets with DLinear. Each dataset-horizon pair trains one full STEPS solver; ablated variants are evaluated by disabling components at test time. Each entry reports average MSE with average relative improvement over the corresponding zero-shot backbone in parentheses.}
\label{tab:component_ablation_rfmd}
\scriptsize
\setlength{\tabcolsep}{2.5pt}
\renewcommand{\arraystretch}{0.84}
\resizebox{\textwidth}{!}{
\begin{tabular}{lcccccc}
\toprule
\textbf{Horizon} & \textbf{Zero-Shot} & \textbf{Full STEPS} & \textbf{Local-only} & \textbf{Global-only} & \textbf{w/o Bounded Fusion} & \textbf{w/o Memory} \\
\midrule
$H=96$ & 0.2533 & 0.1822 (+29.32\%) & 0.2074 (+20.36\%) & 0.2191 (+12.63\%) & \textbf{0.1796 (+30.91\%)} & 0.1922 (+26.98\%) \\
$H=192$ & 0.3112 & 0.2433 (+23.08\%) & 0.2780 (+11.77\%) & 0.2688 (+14.16\%) & \textbf{0.2386 (+25.07\%)} & 0.2731 (+14.18\%) \\
$H=336$ & 0.3769 & 0.2730 (+29.32\%) & 0.3500 (+7.73\%) & 0.2931 (+23.62\%) & \textbf{0.2694 (+30.60\%)} & 0.3368 (+11.46\%) \\
$H=720$ & 0.5452 & 0.3782 (+28.38\%) & 0.5198 (+4.84\%) & 0.3979 (+24.72\%) & \textbf{0.3747 (+28.82\%)} & 0.4978 (+5.35\%) \\
\midrule
\textbf{Avg.} & 0.3716 & 0.2692 (+27.52\%) & 0.3388 (+11.17\%) & 0.2947 (+18.78\%) & \textbf{0.2656 (+28.85\%)} & 0.3250 (+14.49\%) \\
\bottomrule
\end{tabular}}
\end{table*}

\begin{table}[!tbp]
\centering
\caption{Prefix-contamination robustness on DLinear with $H=96$, averaged over five completed datasets (ETTh1, ETTh2, ETTm1, ETTm2, and Exchange). Each column reports MSE under a different outlier ratio in the visible prefix; lower is better. Deg. reports the average per-dataset relative MSE degradation from the clean-prefix setting over nonzero contamination ratios.}
\label{tab:prefix_outlier_robustness}
\scriptsize
\setlength{\tabcolsep}{2.5pt}
\renewcommand{\arraystretch}{0.84}
\begin{tabular}{lcccccc}
\toprule
Method & 0\% & 1\% & 5\% & 10\% & 20\% & Deg. \\
\midrule
Zero-Shot & 0.2649 & 0.2649 & 0.2649 & 0.2649 & 0.2649 & 0.00\% \\
TAFAS & 0.2367 & 0.2482 & 0.3048 & 0.3619 & 0.4866 & 34.09\% \\
PETSA & 0.3128 & 0.3337 & 0.4148 & 0.5003 & 0.6440 & 47.75\% \\
COSA & 0.2526 & 0.3657 & 0.4484 & 0.5275 & 0.6933 & 67.70\% \\
STEPS w/o Bound & 0.1921 & 0.1989 & 0.2302 & 0.2738 & 0.3661 & 33.99\% \\
\textbf{STEPS} & 0.1946 & \textbf{0.1988} & \textbf{0.2191} & \textbf{0.2465} & \textbf{0.3015} & \textbf{22.75}\% \\
\bottomrule
\end{tabular}
\end{table}

\begin{table*}[t]
\centering
\caption{Sparse-prefix locality and far-horizon preservation on DLinear with $H=96$ and $k=3$ revealed future points. The support region is $h=1{:}3$; Near MSE is computed on $h=4{:}27$, and Far MSE is computed on $h=73{:}96$. Positive percentages indicate MSE reduction over the zero-shot backbone.}
\label{tab:real_small_support_k3_only}
\scriptsize
\setlength{\tabcolsep}{2.5pt}
\renewcommand{\arraystretch}{0.84}
\resizebox{\textwidth}{!}{
\begin{tabular}{llcccc@{\hspace{9pt}}llcccc}
\toprule
\textbf{Dataset} & \textbf{Method} & \textbf{Near MSE} & \textbf{Near Imp.} & \textbf{Far MSE} & \textbf{Far Imp.}
& \textbf{Dataset} & \textbf{Method} & \textbf{Near MSE} & \textbf{Near Imp.} & \textbf{Far MSE} & \textbf{Far Imp.} \\
\midrule
\multirow{5}{*}{ETTh1} & Zero-Shot & 0.3929 & -- & 0.5288 & -- &
\multirow{5}{*}{ETTm2} & Zero-Shot & 0.1188 & -- & \textbf{0.1835} & -- \\
 & TAFAS & 0.5104 & -29.91\% & 0.6988 & -32.15\% &
 & TAFAS & 0.1108 & +6.73\% & 0.2075 & -13.08\% \\
 & PETSA & 0.4321 & -9.98\% & 0.6128 & -15.89\% &
 & PETSA & 0.1062 & +10.61\% & 0.1953 & -6.43\% \\
 & COSA & 0.5100 & -29.80\% & 0.6988 & -32.15\% &
 & COSA & 0.1108 & +6.73\% & 0.2075 & -13.08\% \\
 & \textbf{STEPS} & \textbf{0.3911} & +0.46\% & \textbf{0.5171} & +2.21\% &
 & \textbf{STEPS} & \textbf{0.1018} & +14.31\% & 0.1873 & -2.07\% \\
\midrule
\multirow{5}{*}{ETTh2} & Zero-Shot & 0.1696 & -- & 0.2946 & -- &
\multirow{5}{*}{Exchange} & Zero-Shot & 0.0359 & -- & 0.1556 & -- \\
 & TAFAS & 0.1815 & -7.02\% & 0.3331 & -13.07\% &
 & TAFAS & \textbf{0.0290} & +19.22\% & \textbf{0.1394} & +10.41\% \\
 & PETSA & 0.1621 & +4.42\% & 0.3074 & -4.34\% &
 & PETSA & 0.0345 & +3.90\% & 0.1408 & +9.51\% \\
 & COSA & 0.1815 & -7.02\% & 0.3331 & -13.07\% &
 & COSA & 0.0324 & +9.75\% & \textbf{0.1394} & +10.41\% \\
 & \textbf{STEPS} & \textbf{0.1577} & +7.02\% & \textbf{0.2788} & +5.36\% &
 & \textbf{STEPS} & 0.0333 & +7.24\% & 0.1459 & +6.23\% \\
\midrule
\multirow{5}{*}{ETTm1} & Zero-Shot & 0.3093 & -- & \textbf{0.3981} & -- &
\multirow{5}{*}{Weather} & Zero-Shot & 0.1336 & -- & 0.2411 & -- \\
 & TAFAS & 0.2800 & +9.47\% & 0.4655 & -16.93\% &
 & TAFAS & 0.1223 & +8.46\% & 0.2593 & -7.55\% \\
 & PETSA & 0.2740 & +11.41\% & 0.4367 & -9.70\% &
 & PETSA & 0.1273 & +4.72\% & 0.2462 & -2.12\% \\
 & COSA & 0.2799 & +9.51\% & 0.4655 & -16.93\% &
 & COSA & 0.1217 & +8.91\% & 0.2583 & -7.13\% \\
 & \textbf{STEPS} & \textbf{0.2644} & +14.52\% & 0.4007 & -0.65\% &
 & \textbf{STEPS} & \textbf{0.1212} & +9.28\% & \textbf{0.2342} & +2.86\% \\
\bottomrule
\end{tabular}}
\end{table*}

\begin{table}[t]
\centering
\caption{Sparse-anchor correction with limited revealed observations on DLinear ($H=96$), averaged over six datasets. The first 36 predicted steps form the support window; only 5--20\% positions are replaced by true observations, while the remaining positions keep zero-shot predictions. MSE is evaluated on $h=37{:}60$; parenthesized percentages report average per-dataset relative MSE reduction over zero-shot.}
\label{tab:sparse_anchor_limited}
\scriptsize
\setlength{\tabcolsep}{2.5pt}
\renewcommand{\arraystretch}{0.84}
\begin{tabular}{lccccc}
\toprule
\textbf{True anchors} & \textbf{Zero-Shot} & \textbf{TAFAS} & \textbf{PETSA} & \textbf{COSA} & \textbf{STEPS} \\
\midrule
5\% (2) & 0.2696 & 0.2590 (+4.56\%) & 0.2468 (+10.06\%) & 0.2822 (-5.10\%) & \textbf{0.2337 (+15.65\%)} \\
10\% (4) & 0.2696 & 0.2548 (+6.52\%) & 0.2470 (+10.04\%) & 0.2908 (-8.30\%) & \textbf{0.2400 (+12.85\%)} \\
20\% (7) & 0.2696 & 0.2484 (+9.42\%) & 0.2415 (+12.24\%) & 0.2445 (+11.84\%) & \textbf{0.2388 (+12.74\%)} \\
\bottomrule
\end{tabular}
\end{table}

\endgroup

\FloatBarrier

\bibliographystyle{unsrtnat}
\bibliography{rrrrr}

\clearpage
\appendix
The appendix is organized as follows. Appendix~A provides the mathematical derivation that supports the Dirichlet Boundary Value formulation used in the main paper. Appendix~B gives the algorithmic realization of STEPS. Appendix~C reports additional experimental analyses, including aggregate statistics, detailed robustness results, sensitivity studies, and module-only efficiency diagnostics.

\section{Dirichlet Boundary Value Interpretation}

This section provides the mathematical interpretation behind STEPS. The main paper describes the solver architecture through Figure~\ref{fig:overall_framework}; here we formalize why that architecture can be viewed as a boundary-constrained error-field estimation problem, and why a smooth correction field is a natural solution class. Figure~\ref{fig:appendix_dirichlet_view} gives the expanded version of the same pipeline: the frozen forecast defines the base trajectory, the revealed prefix fixes boundary values, the local solver approximates a low-energy Dirichlet extension, the global memory branch supplies long-range residual structure, and SMF produces the final bounded correction.

\subsection{Forecasting TTA as Boundary-Constrained Error Recovery}

Let the future prediction horizon be represented as a one-dimensional temporal graph
\begin{equation}
	\mathcal{G}_H=(\mathcal{V},\mathcal{E}), \qquad
	\mathcal{V}=\{1,\ldots,H\}, \qquad
	\mathcal{E}=\{(h,h+1)\}_{h=1}^{H-1}.
\end{equation}
This graph is the discrete temporal manifold on which predictions, targets, and forecasting errors are defined. A frozen backbone produces
\begin{equation}
	\hat{Y}=f_\theta(X), \qquad \theta \ \mathrm{frozen},
\end{equation}
and the unknown full-horizon error field is
\begin{equation}
	R=Y-\hat{Y}\in\mathbb{R}^{H\times d}.
\end{equation}
During test-time rollout, only the first $a$ future observations may be available. They induce the visible prefix error
\begin{equation}
	R_a=Y_{1:a}-\hat{Y}_{1:a}.
\end{equation}
Thus the observed prefix gives the value of the error field on a subset of nodes, while the remaining future error values are unknown. With
\begin{equation}
	B=\{1,\ldots,a\}, \qquad U=\{a+1,\ldots,H\},
\end{equation}
forecasting TTA can be written as estimating $R_U$ given the boundary values $R_B=R_a$. This is the discrete analogue of a Dirichlet Boundary Value Problem: recover an unknown field in a domain from its boundary values.

In the notation of Figure~\ref{fig:appendix_dirichlet_view}, the blue prefix nodes are not extra training labels for the backbone; they are boundary values of the output error field. The grey future nodes are the unknown interior values. STEPS therefore adapts by solving for a correction field in prediction space, not by changing the forecaster parameters.

\subsection{Dirichlet Energy and the Smooth Error Assumption}

The boundary values alone do not identify a unique future error field. A standard way to select a stable extension on a graph is to minimize the discrete Dirichlet energy
\begin{equation}
	\mathcal{E}_{\mathcal{G}}(R)
	=
	\frac{1}{2}
	\sum_{(i,j)\in\mathcal{E}}
	w_{ij}\|R_i-R_j\|_2^2,
	\label{eq:appendix_dirichlet_energy}
\end{equation}
subject to the observed boundary constraint $R_B=R_a$. On the one-dimensional temporal graph used by STEPS, this reduces to penalizing differences between adjacent horizon steps:
\begin{equation}
	\mathcal{E}_{\mathcal{G}}(R)
	=
	\frac{1}{2}
	\sum_{h=1}^{H-1}
	w_h\|R_{h+1}-R_h\|_2^2.
\end{equation}

\begin{proposition}
	Let $R^\star$ minimize Eq.~\eqref{eq:appendix_dirichlet_energy} under the boundary constraint $R_B=R_a$. For every unobserved node, the optimality condition is a graph-Laplacian equation
	\begin{equation}
		L_{UU}R^\star_U=-L_{UB}R_a,
		\label{eq:appendix_laplacian_system}
	\end{equation}
	where $L$ is the graph Laplacian partitioned over boundary nodes $B$ and unknown nodes $U$.
\end{proposition}

\begin{proof}
	For fixed boundary values, only $R_U$ is optimized. The Dirichlet energy can be written as $\mathcal{E}_{\mathcal{G}}(R)=\frac{1}{2}\mathrm{tr}(R^\top L R)$. Differentiating with respect to $R_U$ gives $L_{UU}R_U+L_{UB}R_B=0$. Substituting $R_B=R_a$ yields Eq.~\eqref{eq:appendix_laplacian_system}.
\end{proof}

This result gives the mathematical origin of the smooth error assumption used by STEPS. A low-energy solution suppresses unnecessary variation between neighboring horizon steps; equivalently, the unknown error field is selected as a smooth extension of the observed prefix boundary. STEPS does not claim that every forecasting residual is globally harmonic. Rather, it uses the Dirichlet principle as a stable local model for the residual component that can be inferred from limited test-time evidence.

\begin{figure*}[!t]
\centering
\includegraphics[width=\textwidth]{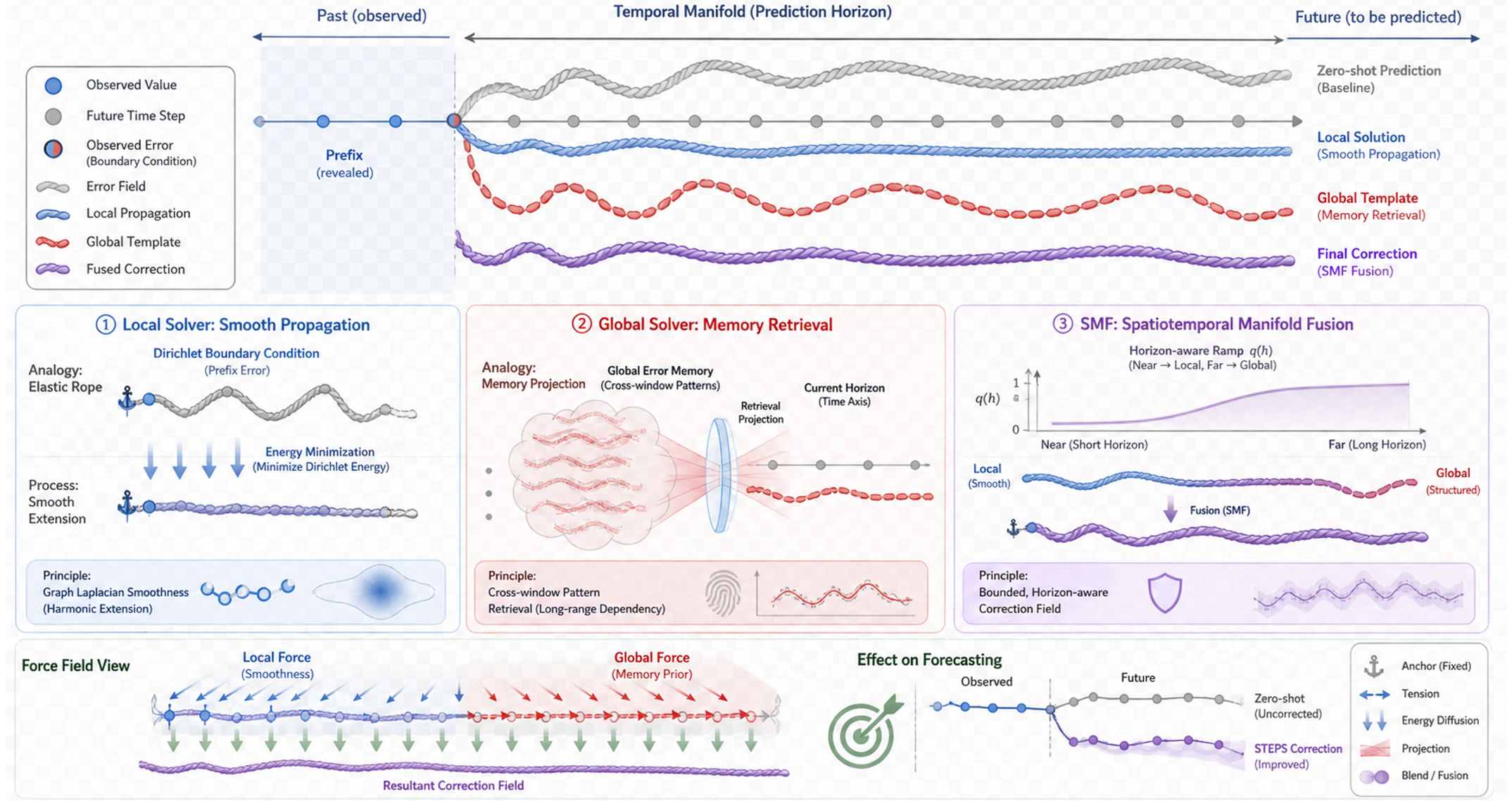}
\caption{Expanded Dirichlet Boundary Value interpretation of STEPS. The frozen backbone gives a base prediction, the revealed prefix error fixes boundary values on the prediction-horizon manifold, Local Solver: Smooth Propagation approximates the low-energy extension supported by nearby boundary evidence, Global Solver: Memory Retrieval supplies recurring long-range residual patterns, and Spatiotemporal Manifold Fusion (SMF) forms the bounded correction field added to the frozen forecast.}
\label{fig:appendix_dirichlet_view}
\end{figure*}

\FloatBarrier
\subsection{Why STEPS Uses Local Propagation, Memory, and Bounded Fusion}

The ideal Dirichlet extension in Eq.~\eqref{eq:appendix_laplacian_system} explains how boundary evidence propagates through the temporal graph, but practical forecasting TTA has three additional constraints that correspond to the three non-backbone modules in Figure~\ref{fig:appendix_dirichlet_view}. First, the observed boundary may be short and noisy, so the solver must select reliable prefix evidence rather than copying every visible fluctuation. Second, distant horizons are weakly constrained by the current prefix, so a purely local low-energy extension may miss recurring long-range error templates. Third, boundary perturbations should not be amplified into an unbounded correction field.

STEPS can therefore be understood as a bounded approximate solver for the boundary value problem. Prefix Boundary Selection defines the boundary $R_a$. Local Smooth Propagation computes a low-variation response with the temporal transfer operator
\begin{equation}
	P=(D^\top D+\alpha I)^{-1},
\end{equation}
where $D^\top D$ is the Laplacian of the one-dimensional temporal chain. This is the module in Figure~\ref{fig:appendix_dirichlet_view} that most directly implements the low-Dirichlet-energy prior. Global Error Memory supplies a data-driven prior for the part of the future error field that is not determined by the current prefix, and the decoder conditions this prior on the current forecast, masked prefix, local response, memory state, and recent context. SMF then combines both responses as
\begin{equation}
	\Delta
	=
	\mathrm{clip}
	\left(
	\Delta_{\mathrm{short}}
	+
	\gamma q(h)\Delta_{\mathrm{long}},
	-c,c
	\right).
\end{equation}
The ramp $q(h)=\sigma(\kappa(h/H-\tau))$ is increasing in the horizon index. Thus near-horizon corrections are dominated by $\Delta_{\mathrm{short}}$, which is most trustworthy close to the boundary, while far-horizon corrections receive more contribution from $\Delta_{\mathrm{long}}$, where memory is needed to recover structure not determined by the prefix. This formulation preserves the boundary-value interpretation while replacing an expensive or underdetermined exact solve with an online correction rule.

The bounded fusion also gives a simple stability property. If the prefix boundary is perturbed by $\epsilon$, a linear local propagation operator changes the response by $P\epsilon$, so the propagated perturbation is controlled by $\|P\|\|\epsilon\|$. Without an output bound, large prefix anomalies may still produce large corrections. With SMF clipping, the final correction satisfies
\begin{equation}
	\|\Delta\|_{\infty}\le c,
\end{equation}
which prevents sparse or contaminated boundary evidence from causing unbounded output-space adaptation.

\section{Algorithmic Realization of STEPS}

Algorithm~\ref{alg:steps} gives the implementation-level procedure corresponding to the mathematical view above. The backbone is frozen throughout the test-time process; adaptation is performed only by constructing and applying the correction field in prediction space.

\begin{algorithm}[htbp]
\caption{STEPS: Smooth Temporal Error Propagation Solver}
\label{alg:steps}
\begin{algorithmic}[1]
\Require Frozen forecaster $f_\theta$, input window $X$, revealed prefix $Y_{1:a}$, Global Error Memory $M_t$
\Ensure Adapted prediction $\hat{Y}_{\mathrm{STEPS}}$
\State Compute zero-shot prediction $\hat{Y}=f_\theta(X)$ with frozen $\theta$
\State Construct prefix error $R_a=Y_{1:a}-\hat{Y}_{1:a}$
\State Select the prefix boundary length and form $(\tilde{R},m)$ from $R_a$
\State Remove low-frequency drift from $R_a$ and obtain $R_{\mathrm{fast}}$
\State Compute local smooth response $\Delta_{\mathrm{harm}}=P_{:,1:a}R_{\mathrm{fast}}$, where $P=(D^\top D+\alpha I)^{-1}$
\State Construct the bias field $\Delta_{\mathrm{bias}}$ and solve the prefix ridge fit to obtain $\Delta_{\mathrm{short}}$
\State Decode global response $\Delta_{\mathrm{long}}=g_\phi(\hat{Y},\Delta_{\mathrm{short}},\tilde{R},m,M_t,z_t)$
\State Fuse local and global responses:
\[
	\Delta=\mathrm{clip}
	\left(
	\Delta_{\mathrm{short}}+\gamma q(h)\Delta_{\mathrm{long}},
	-c,c
	\right)
\]
\State Return $\hat{Y}_{\mathrm{STEPS}}=\hat{Y}+\Delta$
\State After the corresponding targets are observed, update $M_t=\rho M_{t-1}+(1-\rho)\bar{R}_t$
\end{algorithmic}
\end{algorithm}

\section{Additional Experiments and Diagnostics}
The following experiments provide additional diagnostics for the STEPS design. They are kept in the appendix because they analyze implementation choices, parameter sensitivity, normalization compatibility, and module-only online efficiency rather than serving as the primary empirical claims.

\subsection{Additional Aggregate Analysis of Table~\ref{tab:main_results_4_backbones_steps}}

This subsection re-aggregates Table~\ref{tab:main_results_4_backbones_steps} from complementary perspectives. No new MSE values are introduced here; all statistics are computed directly from the main table. We focus on aggregate views where STEPS is the strongest method and avoid repeating diagnostics already shown in the main text.

\begin{table}[!t]
\centering
\caption{Horizon-wise average relative MSE improvement over the zero-shot backbone, aggregated from Table~\ref{tab:main_results_4_backbones_steps}. We report horizons where STEPS obtains the largest average improvement, together with the overall average across all 96 settings.}
\label{tab:table1_horizon_aggregate}
\footnotesize
\setlength{\tabcolsep}{3.6pt}
\renewcommand{\arraystretch}{0.86}
\begin{tabular}{lccccc}
\toprule
\textbf{Horizon} & \textbf{TAFAS} & \textbf{PETSA} & \textbf{COSA-F} & \textbf{COSA-P} & \textbf{STEPS} \\
\midrule
96 & +2.65\% & +1.95\% & +5.94\% & +5.31\% & \textbf{+28.95\%} \\
192 & +2.98\% & +2.76\% & +7.62\% & +6.78\% & \textbf{+23.39\%} \\
336 & +4.96\% & +4.51\% & +13.65\% & +10.34\% & \textbf{+27.77\%} \\
\midrule
Overall & +3.95\% & +3.53\% & +14.05\% & +11.71\% & \textbf{+26.82\%} \\
\bottomrule
\end{tabular}
\end{table}

\begin{figure}[!t]
\centering
\includegraphics[width=\linewidth]{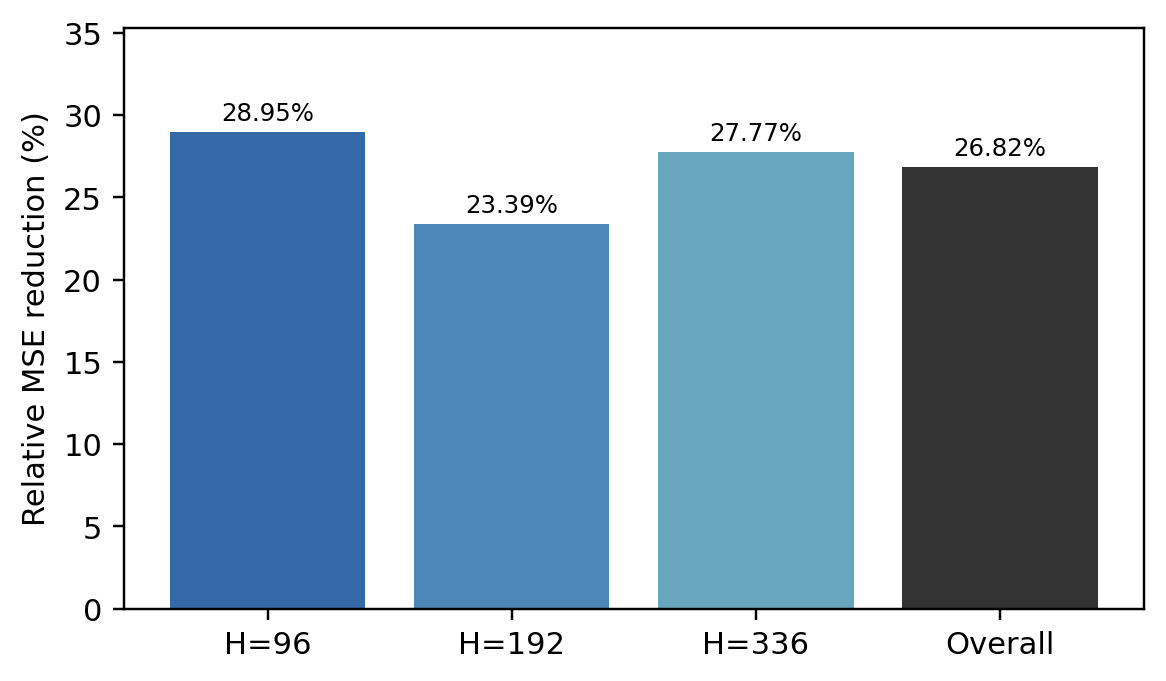}
\caption{Aggregate improvement corresponding to Table~\ref{tab:table1_horizon_aggregate}. The plot shows horizons where STEPS obtains the largest average MSE reduction ($H=96$, $H=192$, and $H=336$), together with the overall average across all 96 settings.}
\label{fig:table1_horizon_aggregate_bar}
\end{figure}

\begin{figure*}[!t]
\centering
\begin{minipage}{0.48\textwidth}
\centering
\includegraphics[width=\linewidth]{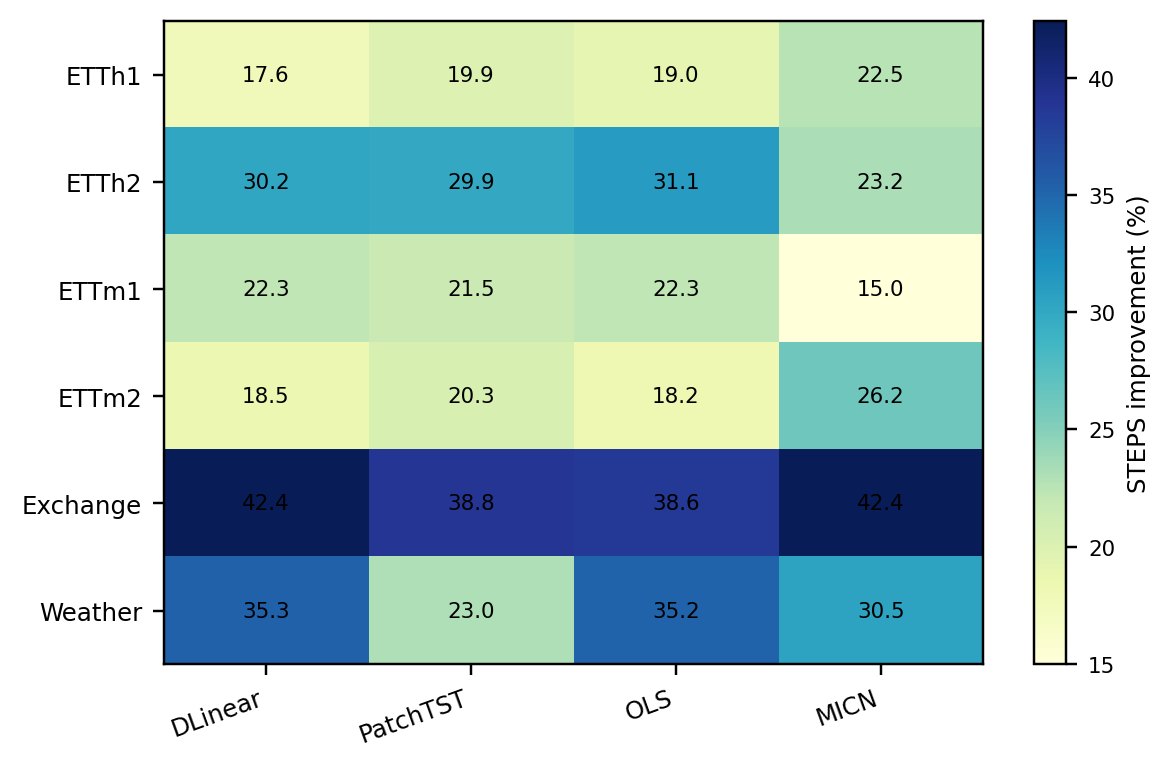}
\end{minipage}
\hfill
\begin{minipage}{0.48\textwidth}
\centering
\includegraphics[width=\linewidth]{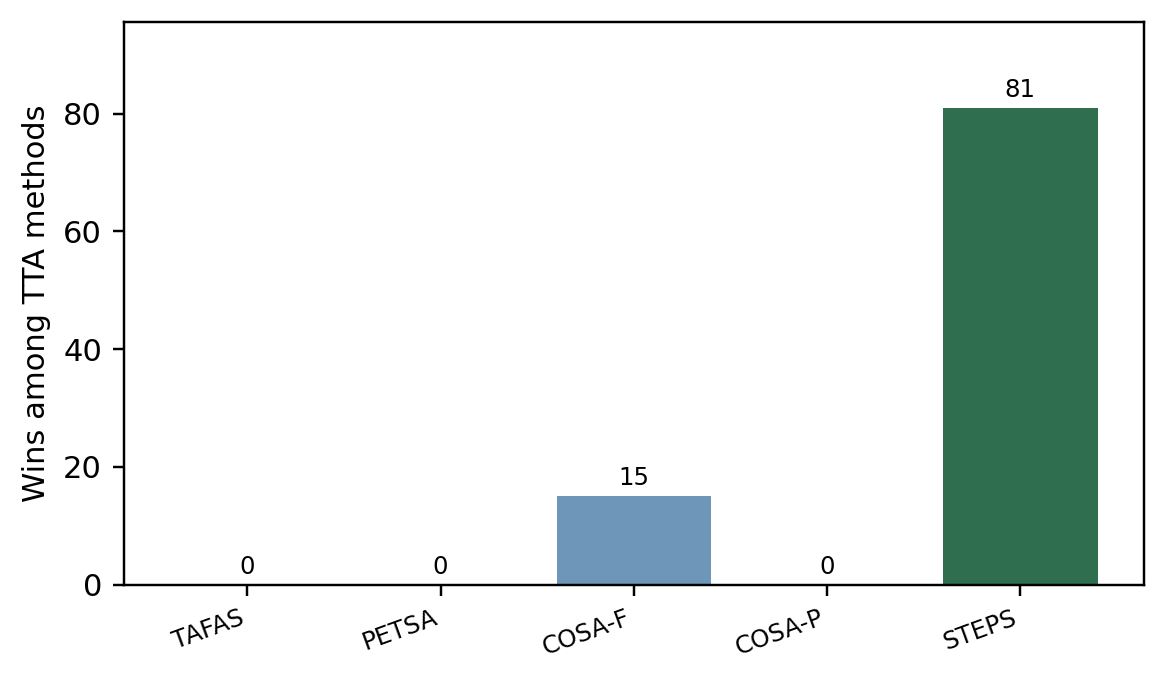}
\end{minipage}
\caption{Dataset/backbone and win-count views of Table~\ref{tab:main_results_4_backbones_steps}. Left: average STEPS improvement by dataset and frozen backbone. Right: number of settings where each TTA method obtains the lowest MSE among the compared TTA methods.}
\label{fig:appendix_table1_heatmap_wins}
\end{figure*}

\begin{table}[!t]
\centering
\caption{Average relative MSE improvement over the zero-shot backbone by frozen backbone, computed from Table~\ref{tab:main_results_4_backbones_steps}. Gap reports STEPS minus the strongest compared TTA baseline under the same backbone.}
\label{tab:appendix_table1_backbone_summary}
\footnotesize
\setlength{\tabcolsep}{3.0pt}
\renewcommand{\arraystretch}{0.88}
\begin{tabular}{lcccccc}
\toprule
Backbone & TAFAS & PETSA & COSA-F & COSA-P & STEPS & Gap \\
\midrule
DLinear & 3.95\% & 3.90\% & 13.57\% & 11.18\% & \textbf{27.71\%} & 14.14 pp \\
PatchTST & 2.31\% & 1.89\% & 13.59\% & 11.18\% & \textbf{25.57\%} & 11.98 pp \\
OLS & 4.16\% & 4.02\% & 13.45\% & 11.11\% & \textbf{27.40\%} & 13.96 pp \\
MICN & 5.39\% & 4.34\% & 15.60\% & 13.38\% & \textbf{26.62\%} & 11.01 pp \\
\bottomrule
\end{tabular}
\end{table}

\begin{table}[!t]
\centering
\caption{Dataset-wise Table~\ref{tab:main_results_4_backbones_steps} summary. STEPS Imp. is the average relative MSE improvement over the zero-shot backbone across four horizons and four frozen backbones. Wins counts the number of dataset-specific settings where STEPS attains the lowest MSE among TTA methods.}
\label{tab:appendix_table1_dataset_summary}
\footnotesize
\setlength{\tabcolsep}{3.2pt}
\renewcommand{\arraystretch}{0.88}
\begin{tabular}{lcccc}
\toprule
Dataset & STEPS Imp. & Best Baseline & Gap & Wins \\
\midrule
ETTh1 & \textbf{19.75\%} & COSA-P (8.90\%) & 10.85 pp & 16/16 \\
ETTh2 & \textbf{28.60\%} & COSA-F (9.50\%) & 19.10 pp & 15/16 \\
ETTm1 & \textbf{20.24\%} & COSA-F (11.16\%) & 9.08 pp & 11/16 \\
ETTm2 & \textbf{20.79\%} & COSA-F (9.56\%) & 11.23 pp & 12/16 \\
Exchange & \textbf{40.55\%} & COSA-F (33.43\%) & 7.12 pp & 12/16 \\
Weather & \textbf{31.01\%} & COSA-F (13.37\%) & 17.64 pp & 15/16 \\
\bottomrule
\end{tabular}
\end{table}

\begin{table}[!t]
\centering
\caption{Horizon-wise win count and average rank over Table~\ref{tab:main_results_4_backbones_steps}. Wins are counted among the five TTA methods at each horizon; rank is computed among TTA methods only, with lower values better.}
\label{tab:appendix_table1_rank_summary}
\footnotesize
\setlength{\tabcolsep}{3.4pt}
\renewcommand{\arraystretch}{0.88}
\begin{tabular}{lcccccc}
\toprule
Method & W@96 & W@192 & W@336 & W@720 & Avg. Rank & Avg. Imp. \\
\midrule
TAFAS & 0 & 0 & 0 & 0 & 4.26 & 3.95\% \\
PETSA & 0 & 0 & 0 & 0 & 4.65 & 3.53\% \\
COSA-F & 0 & 0 & 1 & 14 & 2.10 & 14.05\% \\
COSA-P & 0 & 0 & 0 & 0 & 2.73 & 11.71\% \\
\textbf{STEPS} & \textbf{24} & \textbf{24} & \textbf{23} & \textbf{10} & \textbf{1.26} & \textbf{26.82\%} \\
\bottomrule
\end{tabular}
\end{table}

The aggregate statistics show that STEPS is the strongest method on average across datasets and backbones. The horizon-wise summary further separates this trend: STEPS wins every completed setting at $H=96$ and $H=192$, remains dominant at $H=336$, and is still competitive at $H=720$, where several datasets favor stronger long-range baselines. This pattern is consistent with the proposed boundary-value view: when the visible prefix provides reliable local evidence, STEPS extracts the strongest correction signal; as the horizon becomes very distant, global memory and bounded fusion keep the method stable rather than allowing aggressive over-correction.

\subsection{Additional Backbone Extension}

To test whether the output-space solver remains useful beyond the four main backbones, we add three recent forecasting architectures to the local benchmark: iTransformer, SOFTS, and TimeMixer. This diagnostic is intentionally lightweight and is not used for the main empirical claim: each new backbone is trained for one checkpoint epoch on ETTh1, and STEPS is evaluated with the same quick adaptation budget used in the smoke tests. Table~\ref{tab:appendix_new_backbone_extension} shows that STEPS consistently reduces both MSE and MAE across all tested horizons, including the added $H=192$ setting.

\begin{table*}[!t]
\centering
\caption{Additional backbone extension on ETTh1. Each row reports zero-shot MSE/MAE, STEPS MSE/MAE, and relative MSE reduction. The experiment is a lightweight appendix diagnostic with one checkpoint epoch and one STEPS fitting epoch.}
\label{tab:appendix_new_backbone_extension}
\scriptsize
\setlength{\tabcolsep}{3.0pt}
\renewcommand{\arraystretch}{0.88}
\resizebox{\textwidth}{!}{
\begin{tabular}{llccccc}
\toprule
Backbone & Horizon & Zero-Shot MSE & STEPS MSE & MSE Imp. & Zero-Shot MAE & STEPS MAE \\
\midrule
\multirow{4}{*}{iTransformer} & 96  & 0.5183 & 0.4768 & 8.01\% & 0.5012 & 0.4766 \\
 & 192 & 0.5854 & 0.5555 & 5.10\% & 0.5417 & 0.5244 \\
 & 336 & 0.6683 & 0.6445 & 3.55\% & 0.5850 & 0.5724 \\
 & 720 & 0.8802 & 0.8496 & 3.47\% & 0.6938 & 0.6835 \\
\midrule
\multirow{4}{*}{SOFTS} & 96  & 0.4838 & 0.4465 & 7.70\% & 0.4775 & 0.4554 \\
 & 192 & 0.5462 & 0.5185 & 5.09\% & 0.5197 & 0.5041 \\
 & 336 & 0.6174 & 0.5934 & 3.89\% & 0.5620 & 0.5498 \\
 & 720 & 0.8038 & 0.7723 & 3.92\% & 0.6628 & 0.6524 \\
\midrule
\multirow{4}{*}{TimeMixer} & 96  & 1.6462 & 1.5510 & 5.78\% & 0.8854 & 0.8527 \\
 & 192 & 1.6081 & 1.5323 & 4.72\% & 0.8927 & 0.8667 \\
 & 336 & 3.6526 & 3.4951 & 4.31\% & 1.2738 & 1.2496 \\
 & 720 & 2.0586 & 2.0331 & 1.24\% & 1.0259 & 1.0225 \\
\bottomrule
\end{tabular}}
\end{table*}

The gains are smaller than the fully tuned main-table results, which is expected because this appendix run uses short backbone training and a minimal STEPS fitting budget. Even under this limited setting, the correction remains directionally stable: improvements are strongest at short and medium horizons, while $H=720$ still benefits but leaves less room for prefix-driven correction.

\subsection{Reproducibility, Compute, and Responsible Release Details}

\paragraph{Compute resources.}
Unless otherwise stated, the reported experiments were run with PyTorch 2.5.1 and CUDA 12.1 on a single NVIDIA GeForce RTX 4090 GPU with 24GB memory. The main sweeps are executed sequentially over dataset, horizon, and frozen-backbone settings. After a frozen checkpoint is available, the STEPS fitting and evaluation stage usually takes minutes per setting; end-to-end wall-clock time is dominated by backbone checkpoint training and varies with the backbone architecture and horizon. The module-only runtime in Table~\ref{tab:latency_safety_rfmd} isolates the online adaptation overhead after the frozen backbone prediction.

\paragraph{Implementation details and hyperparameters.}
Table~\ref{tab:appendix_steps_hyperparameters} summarizes the STEPS hyperparameters used across experiments unless a diagnostic explicitly changes them. The same values are shared across datasets, horizons, and frozen backbones. The global decoder is trained with AdamW, learning rate $5\times 10^{-4}$, weight decay $10^{-4}$, gradient clipping at 1.0, 5 fitting epochs, and at most 16 training batches per setting in the main sweep. The frozen forecasting backbones use the dataset splits and checkpoint settings described in the experimental setup.

\begin{table}[!t]
\centering
\caption{Default STEPS hyperparameters used in the main experiments.}
\label{tab:appendix_steps_hyperparameters}
\footnotesize
\setlength{\tabcolsep}{4.5pt}
\renewcommand{\arraystretch}{0.9}
\begin{tabular}{llc}
\toprule
Component & Hyperparameter & Value \\
\midrule
Prefix selection & Minimum prefix support & 2 \\
Local solver & Smoothness $\alpha$ & 0.15 \\
Local solver & Ridge coefficient & 0.03 \\
Local solver & Basis coefficient clip & 0.5 \\
Local solver & Local response mix & 0.55 \\
Global memory & Context size & 8 \\
Global memory & Hidden dimension & 256 \\
Global memory & Memory decay $\rho$ & 0.5 \\
Global memory & Global response scale & 1.5 \\
SMF & Global mixing $\gamma$ & 0.7 \\
SMF & Ramp midpoint $\tau$ & 0.25 \\
SMF & Ramp sharpness $\kappa$ & 8.0 \\
SMF & Final correction clip & 2.5 \\
\bottomrule
\end{tabular}
\end{table}

\paragraph{Statistical reporting.}
The main tables report single-run results under the standard long-horizon forecasting protocol, which is common in this benchmark family because the full grid spans six datasets, four horizons, four frozen backbones, and multiple TTA baselines. To check whether the STEPS fitting stage is sensitive to random initialization, we additionally run a representative three-seed diagnostic on ETTh1 with a fixed DLinear backbone checkpoint. Only the STEPS global decoder is retrained across seeds, while the frozen backbone, dataset split, prefix protocol, and hyperparameters are kept unchanged. Table~\ref{tab:seed_stability_etth1_dlinear} reports mean $\pm$ standard deviation over the three seeds. The variance is small and the mean performance remains close to the corresponding single-run results used elsewhere in the appendix.

\begin{table}[!t]
\centering
\caption{Representative seed-stability diagnostic on ETTh1/DLinear. STEPS is evaluated over three random seeds while keeping the frozen backbone checkpoint fixed. Imp. denotes relative MSE reduction over the zero-shot forecast.}
\label{tab:seed_stability_etth1_dlinear}
\footnotesize
\setlength{\tabcolsep}{5pt}
\renewcommand{\arraystretch}{0.95}
\begin{tabular}{ccccc}
\toprule
Horizon & Zero-shot MSE & STEPS MSE & STEPS MAE & Imp. \\
\midrule
96  & 0.4695 & $0.3753 \pm 0.0001$ & $0.4084 \pm 0.0001$ & 20.06\% \\
192 & 0.5213 & $0.4366 \pm 0.0006$ & $0.4540 \pm 0.0004$ & 16.26\% \\
\bottomrule
\end{tabular}
\end{table}

\paragraph{Broader impact and deployment caution.}
STEPS may improve reliability in forecasting systems deployed under temporal distribution shift, including energy management, weather-aware operations, industrial monitoring, and financial time-series analysis. At the same time, more adaptive forecasts can still be wrong under severe shifts, poor sensors, or adversarially contaminated observations. For safety-critical, financial, or infrastructure decisions, STEPS should be used with monitoring, uncertainty checks, and human oversight rather than as an autonomous decision mechanism.

\paragraph{Existing assets, licenses, and anonymized release.}
All datasets and forecasting backbones used in the experiments are publicly available and are cited in the main paper. The supplementary archive is prepared as an anonymized research artifact with a README, requirements file, reproduction scripts, and license/notice files; local machine paths and author-identifying repository links are removed from release metadata. No new dataset or high-risk model is introduced by this work.

\FloatBarrier
\subsection{Detailed Shared-Checkpoint Ablation Results}

Table~\ref{tab:component_ablation_detail} expands the shared-checkpoint component ablation by dataset and horizon. For each dataset-horizon pair, the full STEPS solver is trained once, and all variants are evaluated by disabling components at test time.

\begin{table*}[!t]
\centering
\caption{Detailed shared-checkpoint component ablation on DLinear. Each entry reports MSE with relative improvement over the corresponding zero-shot backbone in parentheses.}
\label{tab:component_ablation_detail}
\tiny
\setlength{\tabcolsep}{1.8pt}
\renewcommand{\arraystretch}{0.78}
\resizebox{\textwidth}{!}{
\begin{tabular}{llcccccc}
\toprule
\textbf{Dataset} & \textbf{H} & \textbf{Zero-Shot} & \textbf{Full STEPS} & \textbf{Local-only} & \textbf{Global-only} & \textbf{w/o Bound} & \textbf{w/o Memory} \\
\midrule
\multirow{4}{*}{ETTh1} & 96 & 0.4695 & 0.3755 (+20.03\%) & 0.4371 (+6.91\%) & 0.4029 (+14.19\%) & 0.3722 (+20.73\%) & 0.4297 (+8.49\%) \\
 & 192 & 0.5213 & 0.4389 (+15.81\%) & 0.5011 (+3.87\%) & 0.4580 (+12.14\%) & 0.4367 (+16.23\%) & 0.5064 (+2.85\%) \\
 & 336 & 0.5659 & 0.4719 (+16.61\%) & 0.5518 (+2.49\%) & 0.4878 (+13.81\%) & 0.4736 (+16.32\%) & 0.5429 (+4.06\%) \\
 & 720 & 0.7117 & 0.5792 (+18.62\%) & 0.7011 (+1.50\%) & 0.5940 (+16.55\%) & 0.5818 (+18.26\%) & 0.8354 (-17.37\%) \\
\midrule
\multirow{4}{*}{ETTh2} & 96 & 0.2323 & 0.1568 (+32.51\%) & 0.1823 (+21.54\%) & 0.1914 (+17.62\%) & 0.1518 (+34.65\%) & 0.1586 (+31.75\%) \\
 & 192 & 0.2862 & 0.2054 (+28.23\%) & 0.2489 (+13.03\%) & 0.2299 (+19.67\%) & 0.2006 (+29.89\%) & 0.2296 (+19.77\%) \\
 & 336 & 0.3252 & 0.2299 (+29.30\%) & 0.2953 (+9.21\%) & 0.2501 (+23.10\%) & 0.2250 (+30.82\%) & 0.2834 (+12.86\%) \\
 & 720 & 0.4087 & 0.2908 (+28.86\%) & 0.3827 (+6.37\%) & 0.3093 (+24.33\%) & 0.2855 (+30.16\%) & 0.3669 (+10.24\%) \\
\midrule
\multirow{4}{*}{ETTm1} & 96 & 0.3715 & 0.2560 (+31.08\%) & 0.2816 (+24.19\%) & 0.3282 (+11.65\%) & 0.2601 (+29.98\%) & 0.2578 (+30.61\%) \\
 & 192 & 0.4438 & 0.3558 (+19.83\%) & 0.3854 (+13.16\%) & 0.4027 (+9.27\%) & 0.3526 (+20.54\%) & 0.3903 (+12.06\%) \\
 & 336 & 0.5183 & 0.4117 (+20.57\%) & 0.4743 (+8.49\%) & 0.4416 (+14.79\%) & 0.4122 (+20.48\%) & 0.4883 (+5.80\%) \\
 & 720 & 0.5929 & 0.4913 (+17.13\%) & 0.5626 (+5.11\%) & 0.5169 (+12.81\%) & 0.5015 (+15.41\%) & 0.6236 (-5.19\%) \\
\midrule
\multirow{4}{*}{ETTm2} & 96 & 0.1598 & 0.1203 (+24.72\%) & 0.1202 (+24.78\%) & 0.1525 (+4.55\%) & 0.1151 (+27.96\%) & 0.1204 (+24.66\%) \\
 & 192 & 0.1930 & 0.1733 (+10.17\%) & 0.1648 (+14.60\%) & 0.1975 (-2.37\%) & 0.1679 (+13.00\%) & 0.1797 (+6.87\%) \\
 & 336 & 0.2324 & 0.1856 (+20.13\%) & 0.2092 (+10.00\%) & 0.2029 (+12.71\%) & 0.1815 (+21.90\%) & 0.2308 (+0.69\%) \\
 & 720 & 0.3062 & 0.2546 (+16.86\%) & 0.2879 (+5.97\%) & 0.2673 (+12.71\%) & 0.2538 (+17.09\%) & 0.3267 (-6.71\%) \\
\midrule
\multirow{4}{*}{Exchange} & 96 & 0.0913 & 0.0653 (+28.54\%) & 0.0705 (+22.81\%) & 0.0851 (+6.85\%) & 0.0627 (+31.33\%) & 0.0650 (+28.86\%) \\
 & 192 & 0.1827 & 0.1244 (+31.88\%) & 0.1594 (+12.76\%) & 0.1420 (+22.23\%) & 0.1166 (+36.19\%) & 0.1383 (+24.26\%) \\
 & 336 & 0.3277 & 0.1471 (+55.09\%) & 0.3018 (+7.89\%) & 0.1667 (+49.11\%) & 0.1366 (+58.31\%) & 0.2182 (+33.42\%) \\
 & 720 & 0.8873 & 0.4212 (+52.53\%) & 0.8363 (+5.74\%) & 0.4580 (+48.38\%) & 0.3940 (+55.59\%) & 0.4952 (+44.19\%) \\
\midrule
\multirow{4}{*}{Weather} & 96 & 0.1954 & 0.1192 (+39.00\%) & 0.1526 (+21.94\%) & 0.1545 (+20.94\%) & 0.1157 (+40.79\%) & 0.1221 (+37.53\%) \\
 & 192 & 0.2403 & 0.1621 (+32.55\%) & 0.2086 (+13.17\%) & 0.1826 (+23.99\%) & 0.1573 (+34.55\%) & 0.1940 (+19.24\%) \\
 & 336 & 0.2918 & 0.1919 (+34.24\%) & 0.2676 (+8.31\%) & 0.2096 (+28.18\%) & 0.1874 (+35.79\%) & 0.2570 (+11.92\%) \\
 & 720 & 0.3643 & 0.2322 (+36.28\%) & 0.3485 (+4.34\%) & 0.2422 (+33.53\%) & 0.2318 (+36.38\%) & 0.3389 (+6.97\%) \\
\bottomrule
\end{tabular}}
\end{table*}

\FloatBarrier
\subsection{Detailed Prefix-Contamination Results}

Table~\ref{tab:prefix_outlier_robustness_detail} extends the prefix-contamination diagnostic to DLinear, OLS, PatchTST, and MICN under the same $H=96$ protocol. The table averages the completed ETT runs for each backbone, so it measures whether the bounded correction field remains robust beyond a single forecasting architecture.

\begin{table*}[!t]
\centering
\caption{Prefix-contamination robustness across frozen backbones with $H=96$. Results are averaged over completed ETT datasets for each backbone. Each entry reports MSE under a visible-prefix outlier ratio; Deg. is the average degradation over nonzero ratios relative to the clean-prefix setting.}
\label{tab:prefix_outlier_robustness_detail}
\tiny
\setlength{\tabcolsep}{2.5pt}
\renewcommand{\arraystretch}{0.84}
\resizebox{\textwidth}{!}{
\begin{tabular}{llcccccc}
\toprule
\textbf{Backbone} & \textbf{Method} & \textbf{0\%} & \textbf{1\%} & \textbf{5\%} & \textbf{10\%} & \textbf{20\%} & \textbf{Deg.} \\
\midrule
\multirow{6}{*}{DLinear} & Zero-Shot & 0.3083 & 0.3083 & 0.3083 & 0.3083 & 0.3083 & 0.00\% \\
 & TAFAS & 0.2834 & 0.2976 & 0.3674 & 0.4376 & 0.5912 & 49.40\% \\
 & PETSA & 0.3793 & 0.4049 & 0.5039 & 0.6075 & 0.7806 & 51.39\% \\
 & COSA & 1.4400 & 1.5657 & 1.9827 & 2.3866 & 2.9811 & 54.79\% \\
 & STEPS w/o Bound & 0.2249 & 0.2331 & 0.2709 & 0.3236 & 0.4354 & 40.40\% \\
 & \textbf{STEPS} & \textbf{0.2273} & \textbf{0.2325} & \textbf{0.2568} & \textbf{0.2897} & \textbf{0.3556} & \textbf{24.77\%} \\
\midrule
\multirow{6}{*}{OLS} & Zero-Shot & 0.3031 & 0.3031 & 0.3031 & 0.3031 & 0.3031 & 0.00\% \\
 & TAFAS & 0.2736 & 0.2877 & 0.3575 & 0.4284 & 0.5823 & 51.32\% \\
 & PETSA & 0.3644 & 0.3907 & 0.4920 & 0.5980 & 0.7730 & 54.60\% \\
 & COSA & 1.4304 & 1.5564 & 1.9739 & 2.3787 & 2.9744 & 55.26\% \\
 & STEPS w/o Bound & 0.2191 & 0.2275 & 0.2664 & 0.3205 & 0.4346 & 42.52\% \\
 & \textbf{STEPS} & \textbf{0.2216} & \textbf{0.2267} & \textbf{0.2514} & \textbf{0.2847} & \textbf{0.3510} & \textbf{25.68\%} \\
\midrule
\multirow{6}{*}{PatchTST} & Zero-Shot & 0.3085 & 0.3085 & 0.3085 & 0.3085 & 0.3085 & 0.00\% \\
 & TAFAS & 0.2761 & 0.2905 & 0.3604 & 0.4311 & 0.5849 & 50.91\% \\
 & PETSA & 0.3708 & 0.3967 & 0.4967 & 0.5998 & 0.7739 & 52.87\% \\
 & COSA & 1.4323 & 1.5584 & 1.9754 & 2.3810 & 2.9751 & 55.17\% \\
 & STEPS w/o Bound & 0.2179 & 0.2261 & 0.2639 & 0.3169 & 0.4289 & 41.79\% \\
 & \textbf{STEPS} & \textbf{0.2229} & \textbf{0.2280} & \textbf{0.2521} & \textbf{0.3106} & \textbf{0.3508} & \textbf{28.00\%} \\
\midrule
\multirow{6}{*}{MICN} & Zero-Shot & 0.3484 & 0.3484 & 0.3484 & 0.3484 & 0.3484 & 0.00\% \\
 & TAFAS & 0.3175 & 0.3317 & 0.4010 & 0.4720 & 0.6244 & 44.00\% \\
 & PETSA & 0.4210 & 0.4469 & 0.5470 & 0.6476 & 0.8180 & 46.05\% \\
 & COSA & 1.5493 & 1.6652 & 2.0517 & 2.4361 & 3.0118 & 47.88\% \\
 & STEPS w/o Bound & 0.2502 & 0.2586 & 0.2978 & 0.3527 & 0.4692 & 37.71\% \\
 & \textbf{STEPS} & \textbf{0.2518} & \textbf{0.2568} & \textbf{0.2809} & \textbf{0.3141} & \textbf{0.3811} & \textbf{22.39\%} \\
\bottomrule
\end{tabular}}
\end{table*}

Across all four backbones, STEPS keeps the best or near-best clean-prefix MSE while showing the smallest degradation among adaptive methods under increasingly corrupted boundaries. The unbounded variant is often close under clean evidence, but its degradation grows much faster once outliers are injected into the prefix. This supports the role of SMF as a bounded field response: it preserves useful correction while reducing the chance that a short anomalous prefix is amplified into the future horizon.

\FloatBarrier
\subsection{Sparse-Anchor Extension Across Backbones}

The sparse-anchor experiment in the main text asks whether a few revealed observations can define a useful boundary for correction. We extend this diagnostic across frozen backbones and horizons using the most sparse setting, where only two positions in the first 36 prediction steps are replaced by true observations. Table~\ref{tab:appendix_sparse_anchor_backbone_extension} and Figure~\ref{fig:appendix_sparse_anchor_backbone_extension} report completed runs.

\begin{table*}[!t]
\centering
\caption{Sparse-anchor horizon extension with 5\% true anchors. The first 36 prediction steps form the support window, but only two positions are replaced by true observations; MSE is evaluated on $h=37{:}60$. Entries report average relative MSE improvement over the corresponding zero-shot backbone. All entries are completed runs averaged over six datasets.}
\label{tab:appendix_sparse_anchor_backbone_extension}
\scriptsize
\setlength{\tabcolsep}{2.5pt}
\renewcommand{\arraystretch}{0.84}
\begin{tabular}{llcccc}
\toprule
\textbf{Backbone} & \textbf{Method} & \textbf{96} & \textbf{192} & \textbf{336} & \textbf{720} \\
\midrule
\multirow{4}{*}{DLinear} & TAFAS & +4.56\% & +4.55\% & +4.54\% & +4.72\% \\
 & PETSA & +10.06\% & +8.56\% & +6.53\% & +4.50\% \\
 & COSA & -5.10\% & -5.08\% & -5.07\% & -5.14\% \\
 & STEPS & \textbf{+14.50\%} & \textbf{+21.61\%} & \textbf{+23.55\%} & \textbf{+21.75\%} \\
\midrule
\multirow{4}{*}{OLS} & TAFAS & +4.53\% & +4.53\% & +4.52\% & +4.52\% \\
 & PETSA & +9.98\% & +8.51\% & +6.51\% & +4.30\% \\
 & COSA & -5.12\% & -5.10\% & -5.08\% & -5.05\% \\
 & STEPS & \textbf{+12.57\%} & \textbf{+20.08\%} & \textbf{+23.58\%} & \textbf{+20.65\%} \\
\midrule
\multirow{4}{*}{PatchTST} & TAFAS & +4.60\% & +4.49\% & +4.49\% & +4.53\% \\
 & PETSA & +10.16\% & +8.43\% & +6.46\% & +4.30\% \\
 & COSA & -5.17\% & -5.07\% & -5.02\% & -5.08\% \\
 & STEPS & \textbf{+16.85\%} & \textbf{+22.31\%} & \textbf{+23.51\%} & \textbf{+20.25\%} \\
\midrule
\multirow{4}{*}{MICN} & TAFAS & +3.83\% & +3.80\% & +3.79\% & +3.85\% \\
 & PETSA & +8.18\% & +6.91\% & +5.28\% & +3.55\% \\
 & COSA & -4.96\% & -4.92\% & -4.89\% & -4.92\% \\
 & STEPS & \textbf{+11.80\%} & \textbf{+17.35\%} & \textbf{+19.26\%} & \textbf{+17.09\%} \\
\bottomrule
\end{tabular}
\end{table*}

\begin{figure*}[!t]
\centering
\includegraphics[width=0.95\textwidth]{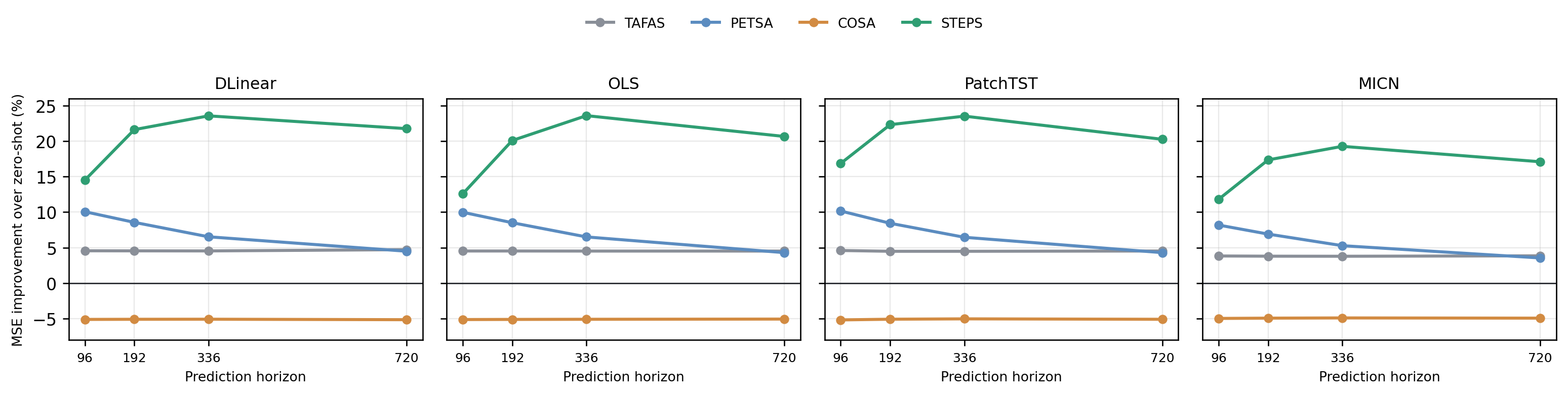}
\caption{Sparse-anchor horizon extension across frozen backbones. With only two true anchors, STEPS remains consistently above the other TTA baselines across horizons.}
\label{fig:appendix_sparse_anchor_backbone_extension}
\end{figure*}

The horizon sweep shows that the sparse-boundary behavior is not specific to DLinear or to $H=96$. Across all completed runs, STEPS gives the largest gain when only two anchors are visible. The gap is most pronounced at longer horizons, where direct prefix fitting has little evidence to determine a stable correction.

\FloatBarrier
\subsection{Compatibility with Normalization}

We further test whether STEPS merely duplicates the effect of input normalization, or whether its output-space error correction remains useful after normalization. We compare DLinear/ETTh1 under two forecasting protocols: \textbf{NoNorm}, which removes the per-window normalization in the forecasting wrapper while retaining dataset-level standardization, and \textbf{NST}, the default Non-stationary Transformer style per-window normalization and de-normalization used in our main experiments. Figure~\ref{fig:normalization_compatibility_mse} reports raw MSE, and Table~\ref{tab:normalization_compatibility_gain} reports the relative MSE reduction from adding STEPS to each corresponding base forecaster.

\begin{figure}[!t]
\centering
\includegraphics[width=\linewidth]{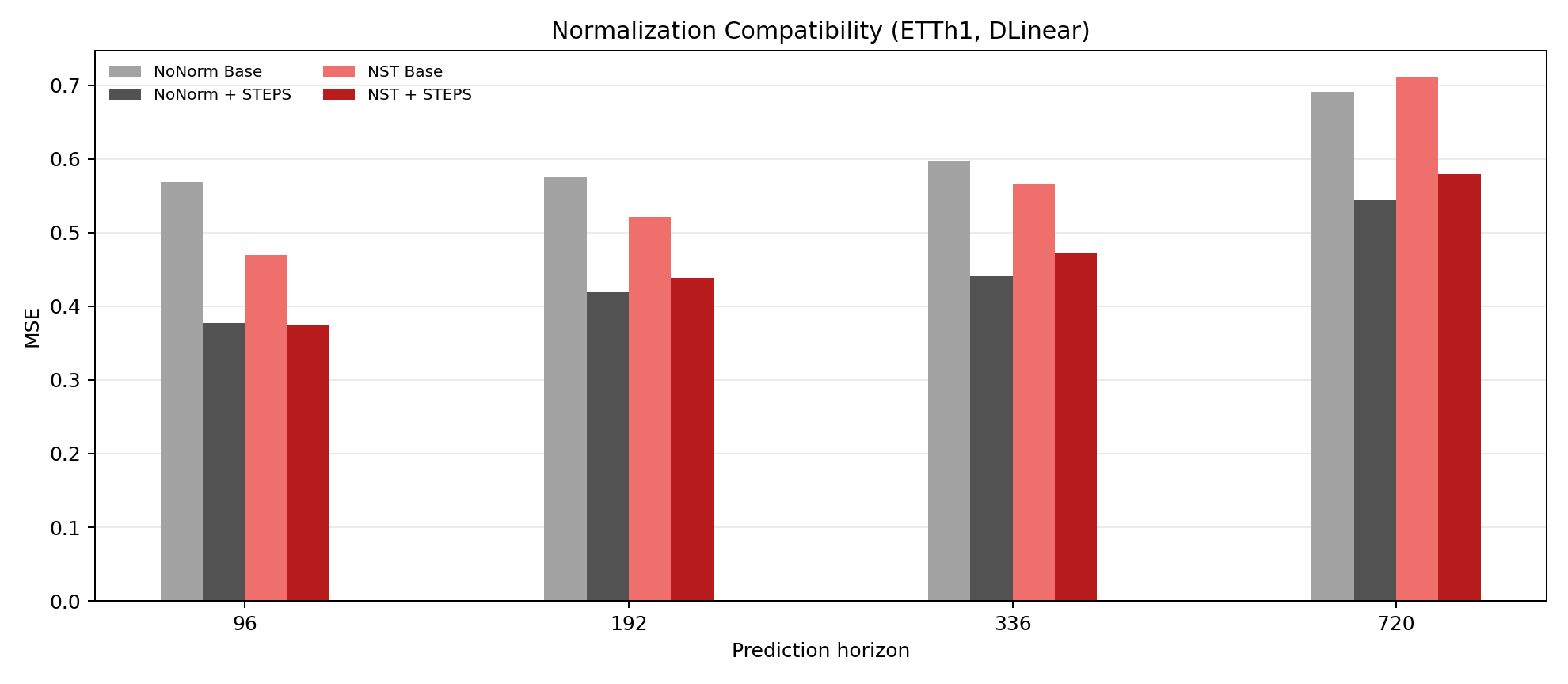}
\caption{Raw MSE under different normalization settings on DLinear/ETTh1. STEPS reduces error under both NoNorm and the default NST protocol, indicating that the error propagation solver is complementary to normalization rather than only replacing mean/scale correction.}
\label{fig:normalization_compatibility_mse}
\end{figure}

\begin{table}[!t]
\centering
\caption{Relative MSE reduction from adding STEPS under different normalization settings on DLinear/ETTh1. Positive values indicate improvement over the corresponding base forecaster.}
\label{tab:normalization_compatibility_gain}
\footnotesize
\setlength{\tabcolsep}{3.2pt}
\renewcommand{\arraystretch}{0.88}
\begin{tabular}{lcc}
\toprule
\textbf{Horizon} & \textbf{NoNorm + STEPS} & \textbf{NST + STEPS} \\
\midrule
$H=96$ & 33.62\% & 20.03\% \\
$H=192$ & 27.25\% & 15.81\% \\
$H=336$ & 26.08\% & 16.61\% \\
$H=720$ & 21.33\% & 18.62\% \\
\bottomrule
\end{tabular}
\end{table}

The improvement remains consistent after NST normalization, especially at the longest horizon where STEPS still reduces MSE by 18.62\%. This suggests that normalization handles part of the level and scale shift, while STEPS corrects remaining horizon-wise error structure from the revealed online prefix error.

\FloatBarrier
\subsection{Global Error Memory Decay Sensitivity}

STEPS maintains Global Error Memory $M_t=\rho M_{t-1}+(1-\rho)\bar{R}_t$. The decay factor $\rho$ controls how quickly the memory changes: small values react quickly to the current error batch, while large values preserve slower error patterns across rollout windows.

\begin{table}[!t]
\centering
\caption{Global Error Memory decay sensitivity on DLinear/ETTh1. Larger $\rho$ gives a slower memory update. Both $H=96$ and $H=720$ are reported to test whether accumulated error patterns remain useful as the horizon grows.}
\label{tab:memory_decay_sensitivity}
\footnotesize
\setlength{\tabcolsep}{3.0pt}
\renewcommand{\arraystretch}{0.88}
\begin{tabular}{llccccc}
\toprule
\textbf{Horizon} & \textbf{Metric} & 0.00 & 0.50 & 0.80 & 0.92 & 0.98 \\
\midrule
$H=96$ & MSE & 0.3847 & \textbf{0.3757} & 0.3828 & 0.4121 & 0.4208 \\
$H=96$ & MAE & 0.4136 & \textbf{0.4087} & 0.4136 & 0.4330 & 0.4381 \\
$H=720$ & MSE & \textbf{0.5690} & 0.5816 & 0.5755 & 0.6028 & 0.7117 \\
$H=720$ & MAE & \textbf{0.5503} & 0.5592 & 0.5515 & 0.5614 & 0.6099 \\
\bottomrule
\end{tabular}
\end{table}

\FloatBarrier
\subsection{Prefix Error Length Sensitivity}

The prefix error acts as a Dirichlet boundary for the local solver. We compare fixed prefix lengths against the FFT-estimated prefix used by STEPS. Very short boundaries are underdetermined and may be dominated by noise, while the adaptive prefix gives the most stable error evidence in this run.

\begin{table}[!t]
\centering
\caption{Prefix error length sensitivity on DLinear/ETTh1 with $H=96$. The prefix error is the visible boundary used by the local propagation solver; lower MSE/MAE is better.}
\label{tab:prefix_length_sensitivity}
\footnotesize
\setlength{\tabcolsep}{3.2pt}
\renewcommand{\arraystretch}{0.88}
\begin{tabular}{lcccccc}
\toprule
\textbf{Prefix} & 1 & 2 & 3 & 5 & 7 & FFT-est. \\
\midrule
MSE & 0.3945 & 0.3926 & 0.3902 & 0.3865 & 0.3830 & \textbf{0.3759} \\
MAE & 0.4237 & 0.4220 & 0.4201 & 0.4173 & 0.4146 & \textbf{0.4088} \\
\bottomrule
\end{tabular}
\end{table}

\FloatBarrier
\subsection{Local Error Smoothness Sensitivity}

The local branch performs smooth error propagation with temporal transfer operator $P=(D^\top D+\alpha I)^{-1}$. The smoothness parameter $\alpha$ controls how aggressively the prefix error is propagated. Very small $\alpha$ over-propagates the observed error boundary, while moderate regularization produces a stable correction.

\begin{table}[!t]
\centering
\caption{Local error smoothness sensitivity on DLinear/ETTh1 with $H=96$. The parameter $\alpha$ regularizes local error propagation; lower MSE/MAE is better.}
\label{tab:transfer_alpha_sensitivity}
\footnotesize
\setlength{\tabcolsep}{3.2pt}
\renewcommand{\arraystretch}{0.88}
\begin{tabular}{lcccccc}
\toprule
\textbf{$\alpha$} & 0.01 & 0.05 & 0.10 & 0.50 & 1.00 & 5.00 \\
\midrule
MSE & 0.3763 & \textbf{0.3757} & 0.3759 & 0.3786 & 0.3820 & 0.3871 \\
MAE & 0.4093 & \textbf{0.4088} & 0.4088 & 0.4106 & 0.4131 & 0.4167 \\
\bottomrule
\end{tabular}
\end{table}

\FloatBarrier
\subsection{Horizon-Aware Fusion Schedule Visualization}

To make the SMF behavior more transparent, we visualize the fusion schedule used by STEPS. The final correction has the form
\[
	\Delta
	=
	\mathrm{clip}
	\left(
	\Delta_{\mathrm{short}}
	+
	\gamma q(h)\Delta_{\mathrm{long}},
	-c,c
	\right),
\]
so the local branch enters with unit coefficient, while the global branch is scaled by $\gamma q(h)$. In this diagnostic, we use the default SMF settings $\gamma=0.7$, $\kappa=8.0$, and $\tau=0.25$. We also report normalized local/global shares under an equal-magnitude response assumption, i.e.,
\[
	w_{\mathrm{local}}(h)=\frac{1}{1+\gamma q(h)},
	\qquad
	w_{\mathrm{global}}(h)=\frac{\gamma q(h)}{1+\gamma q(h)}.
\]
These normalized shares are not extra learned gates; they are a visualization of the relative mixing pressure implied by SMF when local and global response magnitudes are comparable.

\begin{figure}[!t]
\centering
\includegraphics[width=\linewidth]{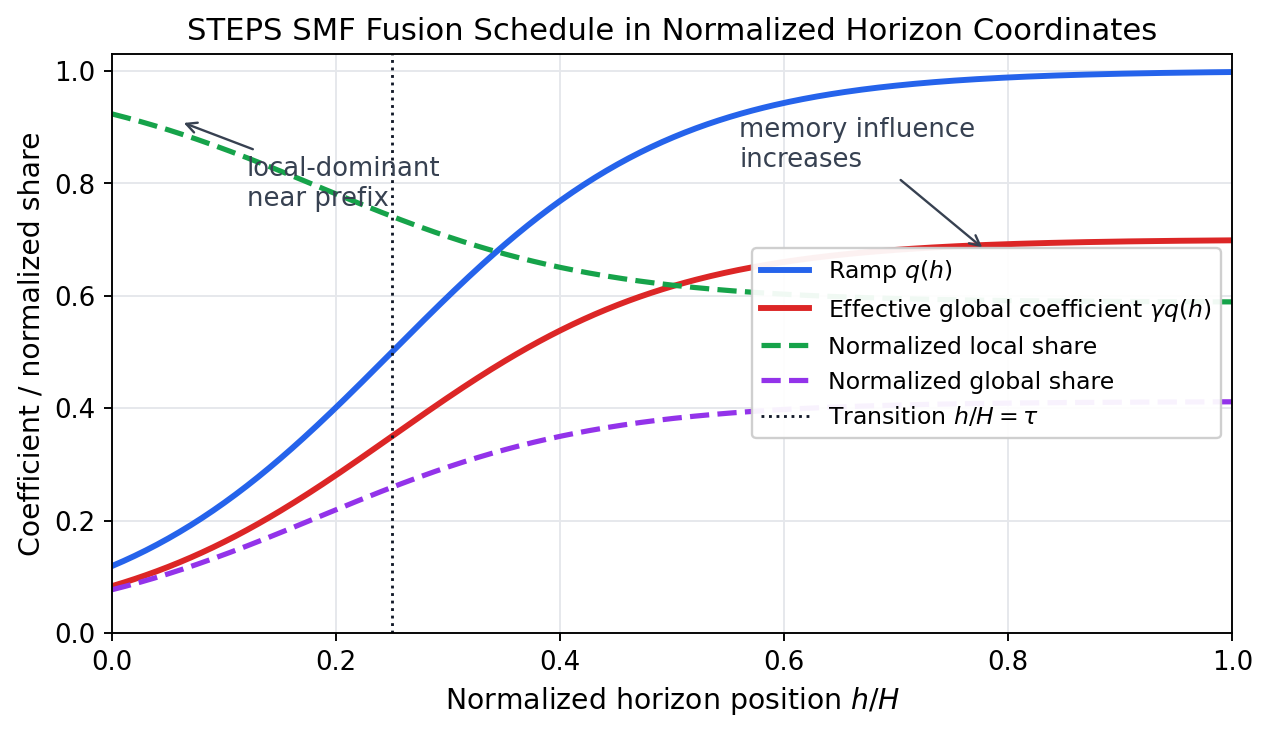}
\caption{SMF fusion schedule in normalized horizon coordinates. Since $q(h)$ is a function of $h/H$, all forecasting horizons share the same relative schedule: local propagation dominates near the prefix, while Global Error Memory becomes increasingly influential toward distant steps.}
\label{fig:steps_fusion_schedule_normalized}
\end{figure}

\begin{table}[!t]
\centering
\caption{Discrete interpretation of the normalized SMF schedule across forecasting horizons. Start, Mid, and End report normalized local/global shares under the equal-magnitude visualization above. The relative shares are almost identical across horizons because the gate uses $h/H$; the absolute transition step $\tau H$ scales with $H$.}
\label{tab:steps_fusion_schedule_horizons}
\footnotesize
\setlength{\tabcolsep}{3.2pt}
\renewcommand{\arraystretch}{0.9}
\begin{tabular}{ccccc}
\toprule
\textbf{Horizon} & \textbf{Transition} & \textbf{Start L/G} & \textbf{Middle L/G} & \textbf{End L/G} \\
 & $\boldsymbol{\tau H}$ & & & \\
\midrule
$H=96$ & 24 & 92.3 / 7.7 & 62.0 / 38.0 & 58.9 / 41.1 \\
$H=192$ & 48 & 92.3 / 7.7 & 61.9 / 38.1 & 58.9 / 41.1 \\
$H=336$ & 84 & 92.3 / 7.7 & 61.9 / 38.1 & 58.9 / 41.1 \\
$H=720$ & 180 & 92.3 / 7.7 & 61.9 / 38.1 & 58.9 / 41.1 \\
\bottomrule
\end{tabular}
\end{table}

Figure~\ref{fig:steps_fusion_schedule_normalized} and Table~\ref{tab:steps_fusion_schedule_horizons} clarify why plotting four separate curves against their own horizon ranges produces nearly identical shapes. This is expected: SMF is intentionally horizon-normalized through $h/H$, so each forecast length follows the same relative mixing policy. What changes across horizons is not the relative curve shape but the absolute step at which the transition occurs. With $\tau=0.25$, the transition point is step 24 for $H=96$, step 48 for $H=192$, step 84 for $H=336$, and step 180 for $H=720$.

This design makes the fusion rule comparable across forecast lengths. Near the beginning, all horizons give a strongly local-dominant split of about $92.3\%/7.7\%$, reflecting that the prefix boundary is most informative nearby. Around the middle, the global share rises to about $38\%$. At the final step, the effective global coefficient approaches $\gamma=0.7$, yielding a normalized local/global split of about $58.9\%/41.1\%$. Thus, STEPS does not abruptly switch from local to global correction; it gradually increases memory influence while retaining a nontrivial local component throughout the horizon.

\FloatBarrier
\subsection{Module-Only Efficiency Diagnostics}

Finally, we evaluate the online computational cost of the STEPS correction module. Since the forecasting backbone is frozen, Table~\ref{tab:latency_safety_rfmd} reports only the additional output-space correction cost after the backbone prediction. Runtime is measured on synthetic prediction windows with batch size 48, $d=7$, and a visible prefix length of 4, excluding the frozen forecasting backbone and data loading. Correction time includes Local Smooth Propagation, Global Error Memory decoding, SMF, and memory update. Throughput is reported as adapted prediction windows per second.

\begin{table*}[!t]
\centering
\caption{Module-only online efficiency of STEPS across forecasting horizons. Runtime is measured with batch size 48 on synthetic prediction windows and excludes the frozen forecasting backbone and data loading. FLOPs count the dominant decoder linear layers per adapted prediction window. Throughput reports adapted prediction windows per second.}
\label{tab:latency_safety_rfmd}
\scriptsize
\setlength{\tabcolsep}{4.0pt}
\renewcommand{\arraystretch}{0.86}
\resizebox{\textwidth}{!}{
\begin{tabular}{ccccccc}
\toprule
\textbf{Horizon} & \textbf{Trainable Params} & \textbf{FLOPs/window} & \textbf{Peak Mem.} & \textbf{Correction Time} & \textbf{Time/window} & \textbf{Throughput} \\
 & & & & \textbf{(ms/batch)} & \textbf{(ms)} & \textbf{(windows/s)} \\
\midrule
$H=96$ & 215.6K & 3.01M & 12.88 MB & 2.42 & 0.050 & 19,835 \\
$H=192$ & 363.2K & 5.08M & 15.70 MB & 2.54 & 0.053 & 18,898 \\
$H=336$ & 584.5K & 8.17M & 20.03 MB & 2.77 & 0.058 & 17,329 \\
$H=720$ & 1.17M & 16.43M & 31.80 MB & 3.42 & 0.071 & 14,035 \\
\bottomrule
\end{tabular}}
\end{table*}

The timing table isolates the STEPS correction overhead rather than end-to-end forecasting cost. Even at $H=720$, the full online correction takes 3.42 ms per batch, or 0.071 ms per adapted prediction window, while requiring no online backpropagation or backbone update. The FLOPs and parameter count increase with the horizon because the memory decoder maps horizon-length fields, but the measured correction latency remains small relative to typical backbone inference and data-loading overhead.

\end{document}